\documentclass{article}
\PassOptionsToPackage{numbers, compress}{natbib}

\usepackage[letterpaper,top=2cm,bottom=2cm,left=3cm,right=3cm,marginparwidth=1.75cm]{geometry}

\usepackage[utf8]{inputenc} 
\usepackage[T1]{fontenc}    
\usepackage{url}            
\usepackage{booktabs}       
\usepackage{nicefrac}       
\usepackage{microtype}      
\usepackage{lipsum}		
\usepackage{natbib}
\setcitestyle{authoryear,open={(},close={)}} 
\usepackage{doi}
\usepackage{amsfonts} 
\usepackage{xcolor}
\usepackage{amsmath}
\usepackage{amssymb}
\usepackage{amsthm}
\usepackage{bbm}
\usepackage{hyperref}
\usepackage{graphicx}
\usepackage{mathtools}
\usepackage[ruled, vlined]{algorithm2e}
\usepackage{textcomp}
\usepackage{adjustbox}
\SetKwInput{KwData}{Data}
\SetKwInput{KwParameters}{Parameters}
\SetKwInput{KwInit}{Initialize}
\usepackage[capitalize,noabbrev]{cleveref}
\newcommand{\footremember}[2]{%
    \thanks{#2}
    \newcounter{#1}
    \setcounter{#1}{\value{footnote}}%
}
\newcommand{\footrecall}[1]{%
    \footnotemark[\value{#1}]%
} 

\usepackage{apptools}
\AtAppendix{\counterwithin{lemma}{section}}
\AtAppendix{\counterwithin{remark}{section}}
\AtAppendix{\counterwithin{theorem}{section}}
\AtAppendix{\counterwithin{figure}{section}}
\AtAppendix{\counterwithin{assumption}{section}}
\title{Fast, Distribution-free Predictive Inference for Neural Networks with Coverage Guarantees}

\author{Yue Gao\footremember{uwm}{Department of Statistics,  University of Wisconsin Madison, Madison, WI, USA.},~
  Garvesh Raskutti\footrecall{uwm} ,~
Rebecca Willet\footremember{uchicagostat}{Department of Statistics,  University of Chicago, Chicago, IL, USA.}\footremember{uchicagocs}{Department of Computer Sciences,  University of Chicago, Chicago, IL, USA.}
}

\usepackage{xspace}
\newcommand{\ind}{\perp\!\!\!\!\perp} 
\newcommand{\cA}{\mathcal{A}}

\newcommand{\cD}{\mathcal{D}}
\newcommand{\cB}{\mathcal{B}}

\newcommand{\cS}{\mathcal{S}}
\newcommand{\cL}{\mathcal{L}}
\newcommand{\cN}{\mathcal{N}}

\newcommand{\bbP}{\mathbb{P}}
\newcommand{\cM}{\mathcal{M}}
\newcommand{\cX}{\mathcal{X}}
\newcommand{\cY}{\mathcal{Y}}
\newcommand{\bX}{\mathbf{X}}

\newcommand{\bz}{\mathbf{z}}
\newcommand{\bY}{\mathbf{Y}}
\newcommand{\eg}{\xspace\textit{e.g.},\xspace}
\newcommand{\ie}{\xspace\textit{i.e.},\xspace}
\newcommand{\bbR}{\mathbb{R}}
\newcommand{\DP}{\text{DP}}
\newcommand{\lazy}{\text{LAZY}}
\newcommand{\bbE}{\mathbb{E}}
\newcommand{\ReLU}{\text{ReLU}}
\newcommand{\given[1]}{\:#1\vert\:}
\newcommand{\bbone}{\mathbbm{1}}
\newcommand{\rmj}{\backslash{j}}

\newcommand{\card}{\text{card}}
\newcommand{\Lip}{\text{Lip}}

\newcommand{\test}{\text{test}}

\newtheorem{definition}{Definition}
\newtheorem{lemma}{Lemma}
\newtheorem{theorem}{Theorem}

\newtheorem{remark}{Remark}
\newtheorem{assumption}{Assumption}
\newcommand{\aka}{\textit{aka. }}
\crefname{algocf}{alg.}{algs.}
\crefname{assumption}{Assumption }{Assumptions}
\Crefname{algocf}{Algorithm}{Algorithms}
\newcommand{\argmin}{\mathop{\mathrm{argmin}}\limits}

\newcommand{\revise}[1]{ #1}

\begin{document}

\maketitle
\begin{abstract}
This paper introduces a novel, computationally-efficient algorithm for predictive inference (PI) that requires no distributional assumptions on the data and can be computed faster than existing bootstrap-type methods for neural networks.
Specifically, if there are $n$ training samples, bootstrap methods require training a model on each of the $n$ subsamples of size $n-1$; for large models like neural networks, this process can be computationally prohibitive. In contrast, 
our proposed method
trains one neural network on the full dataset with $(\epsilon, \delta)$-differential privacy (DP) and then approximates each leave-one-out model efficiently using a linear approximation around the differentially-private neural network estimate.
With exchangeable data, we prove that our approach has a rigorous coverage guarantee that depends on the preset privacy parameters and the stability of the neural network, regardless of the data distribution. Simulations and experiments on real data demonstrate that our method 
satisfies the coverage guarantees with substantially reduced computation compared to bootstrap methods. 
\end{abstract}


\section{Introduction }

To assess the accuracy of parameter estimates or predictions without specific distributional knowledge of the data, the idea of re-sampling or sub-sampling on the available data has been long-established to construct prediction intervals, and there is a rich history in the statistics literature on the jackknife and bootstrap methods, see ~\cite{stineBootstrapPredictionIntervals1985b, efronBootstrapMethodsAnother1979, quenouilleApproximateTestsCorrelation1949, efronLeisurelyLookBootstrap1983}.
Among these re-sampling methods, leave-one-out methods (generally referred to as ``cross-validation'' or ``jackknife'') are 
widely used to assess or calibrate predictive accuracy, and can be found in a large line of literature \citep{stoneCrossValidatoryChoiceAssessment1974b, geisserPredictiveSampleReuse1975}.

While it has been demonstrated in a large body of past work with extensive evidence that jackknife-type methods have reliable empirical performance, the theoretical properties of these types of methods are studied relatively little until recently, see \cite{steinberger2018conditional, bousquetStabilityGeneralization}.
One of the most important results among these theoretically guaranteed works is \cite{foygel2019predictive}, which introduces a crucial modification compared to the traditional jackknife method that permits rigorous coverage guarantees of at least $1 - 2 \alpha$ regardless of the distribution of the data points, for any algorithm that treats the training points symmetrically. 
We will revisit this work and give more relative details in \cref{sec:jackknife+}.

Although theoretically jackknife+ has been proven to have coverage guarantees without distributional assumptions, in practice, this method is computationally costly, since we need to train $n$ (which is the training sample size) leave-one-out models from scratch to find the predictive interval. Especially for large and complicated models like neural networks, this computational cost is  prohibitive. The goal of this paper is to provide a \textit{fast} algorithm that provides similar theoretical coverage guarantees to those in jackknife+.

To achieve this goal, we develop a new procedure, called Differentially Private Lazy Predictive Inference (DP-Lazy PI), which combines two ideas: lazy training of neural networks and differentially private stochcastic gradient descent (DP-SGD).  To accelerate the procedure, we introduce a lazy training scheme inspired by \citet{chizat2020lazy,pmlr-v162-gao22h} to train the leave-one-out models. The intuition is that with data exchangeability, the leave-one-out models should be quite close or similar to each other and there is no need to train each one from scratch (\ie random initialization). Instead, we  first train a model on the full data and use this model as a good initialization. By using DP-SGD as the initializer for our full model, we are able to provide coverage guarantees since the privacy mechanism prevents information leakage across leave-one-out estimators. In particular, we prove that our DP-Lazy PI procedure has a coverage of at least $1- 2\alpha - 3\sqrt{2\eta + 2\epsilon +\delta}$ where $\eta$ represents an out-of-sample stability parameter and $(\epsilon, \delta)$ are the differential privacy parameters. Empirically, we show through simulations and real-data experiments that our method has significant advantage over the existing jackknife+ method in run-time while still maintaining good coverage.

\section{Preliminaries}

We first define key notation. For any values $v_i$ indexed by $i\in [n]$, define 
\begin{equation}\label{def: quantile}
    Q_{\alpha}^{+}\left(\{v_i\}_{i\in[n]}\right) = \lceil (1-\alpha)(n+1)\rceil^{\rm th} \text{ smallest value of } v_1, \dots, v_n,
\end{equation}
\ie the ${1-\alpha}^{\rm th}$ quantile of the empirical distribution of these values; Similarly, 
\begin{equation}    
    Q_{\alpha}^{-}\left(\{v_i\}_{i\in[n]}\right) = \lfloor (\alpha)(n+1)\rfloor^{\rm th} \text{ smallest value of } v_1, \dots, v_n = - Q_{\alpha}^{+}\left(\{-v_i\}_{i\in[n]}\right)
\end{equation}
is the $\alpha^{\rm}$ quantile of the empirical distribution.

\subsection{Distribution-free Prediction Intervals}
\label{section: jackknife+}

Suppose we have training data $(X_i, Y_i)\in \cX\times \cY$ for $i = 1, \dots, n$, and a new test point $(X_{n+1}, Y_{n+1})$, $X_i\in \bbR^p, ~Y_i\in \bbR$ drawn independently from the same distribution. 
We fit a regression model to the training data, \ie a function $\hat{f}: \cX\mapsto \cY$ where $\hat{f}(x)$ predicts $Y_{n+1}$ given a new feature vector $X_{n+1} = x$, and then provide a prediction interval centered around $\hat{f}(X_{n+1})$ for the test point. 
Specifically, given some target coverage level $1-\alpha$, we aim to construct a prediction interval $\hat{C}_{n, \alpha}$, such that 
$\bbP\left[ Y_{n+1} \in \widehat{C}_{n, \alpha}(X_{n+1})\right]   \geq 1- \alpha.$
We call the probability $\bbP\left[ Y_{n+1} \in \widehat{C}_{n, \alpha}(X_{n+1})\right]$ as the coverage, and  $\text{len}[\widehat{C}_{n, \alpha}(X_{n+1})]$  as the interval width of $\widehat{C}_{n, \alpha}(X_{n+1})$, which is the distance between the left and right endpoints.
\paragraph{Na\"ive Prediction Interval}
A na\"ive way to construct a predictive interval at the new test point $(X_{n+1}, Y_{n+1})$ is to center the interval at $\hat{f}(X_{n+1})$ and estimate the margin (\ie half of the interval width) from the training residuals $|Y_i - \hat{f}(X_{i})|, ~i = 1, \dots, n$.
Therefore, a na\"ive prediction interval can be constructed as:
\[
    \widehat{C}_{n,\alpha}^{\text{na\"ive}}(x) \coloneqq 
    \left[
        \hat{f}(x) \pm Q_{\alpha}^+ \left(\{|Y_i - \hat{f}(X_{i})|\}_{i\in [n]}\right)
    \right].
    \]
where $Q_{\alpha}^+ \left(\{|Y_i - \hat{f}(X_{i})|\}_{i\in [n]}\right)$ is defined as \eqref{def: quantile}.
Due to the problem of overfitting when the training errors are typically smaller than the test errors, this na\"ive interval $\widehat{C}_{n, \alpha}(x)$ may  undercover---\ie  the probability that $Y_{n+1}$ falls outside the interval can be larger than $\alpha$, as discussed in \cite{foygel2019predictive}.

\paragraph{Jackknife Prediction Interval}
To avoid undercoverage due to model overfitting, the \emph{jackknife} method estimates the margin of errors 
is estimated by leave-one-out residuals instead of training residuals \citep{steinberger2016leave, steinberger2018conditional}.
The idea is straightforward: for each $j = 1, \dots, n$, we fit a regression function $\hat{f}_{-j}$ using all the training data except the $j$-th training sample. Based on these leave-one-out models $\hat{f}_{-1},\dots, \hat{f}_{-n}$, we can therefore compute the leave-one-out residuals: $|Y_1 - \hat{f}_{-1}(X_1)|, \dots, |Y_n - \hat{f}_{-n}(X_n)|$. 

With the leave-one-out residuals as well as the regression function $\hat{f}$ fitted on the full training data, the jackknife prediction interval is constructed as:
\[
\widehat{C}_{n, \alpha}^{\text{jackknife}}(x)
\coloneqq \left[
    \hat{f}(x) \pm Q_{\alpha}^{+}
    \left(\{|Y_i - \hat{f}_{-i}(X_{i})|\}_{i\in [n]}\right)
\right].
\]

\paragraph{Jackknife+ Prediction Interval}
\label{sec:jackknife+}

The jackknife+ is a modification of jackknife, and both of these methods use the leave-one-out residuals when constructing the prediction intervals. The difference is that the jackknife interval is centered at the predicted value $\hat{f}(X_{n+1})$ (where $\hat{f}$ is fitted on the full training data), whereas jackknife+ uses the leave-one-out predictions $\hat{f}_{-i}(X_{n+1}), i = 1, \dots, n$ instead to the build the prediction interval:
\begin{align*}
    &\widehat{C}^{\text{jackknife+}}_{n, \alpha}(x) \\
    & \qquad \coloneqq
    \left[
    Q_{\alpha}^{-}\left( \left\{
        \hat{f}_{-i}(x) - 
        |Y_i - \hat{f}_{-i}(X_{i})|\right\}_{i\in [n]}\right)
    ,~~
    Q_{\alpha}^{+}\left( \left\{
        \hat{f}_{-i}(x) + 
        |Y_i - \hat{f}_{-i}(X_{i})|\right\}_{i\in [n]}\right)
    \right].
\end{align*}
\citet{foygel2019predictive} provides the non-asymptotic coverage guarantee for jackknife+ without assumptions beyond the training and test data being exchangeable. Theoretically, a distribution-free coverage guarantee of at least $1-2\alpha$ is provided, and it's observed that the method can achieve $1-\alpha$ coverage in empirical studies.
\paragraph{Summary of coverage guarantees and computational costs}

The coverage guarantees and the computational cost  of the above methods are summarized in \cref{table: summary_pi}. Concretely computational cost refers to the computation at two stages, model training and evaluation and we measure cost in terms of number of models trained during the training phase and number of calls to the trained function in the evaluation stage.

\begin{table}[ht]
    \centering
    \caption{Coverage guarantees and computation costs of existing methods for estimating prediction interval. Here $n$ is the training sample size; $n_{\text{test}}$ is the number of test points for which we seek prediction intervals.
    } 
   \adjustbox{max width = \textwidth}{\begin{tabular}{l|cc|cc}
    \hline
              & \multicolumn{2}{c|}{Coverage Guarantee}     & \multicolumn{2}{c}{Computation} \\ \cline{2-5} 
              & Distribution-free Theory & Empirical          & \# models trained          & \# calls to trained function         \\ \hline
    \textbf{Na\"ive    } & No guarantee           & $<1-\alpha$        & $1$                 & $n+n_{\text{test}}$         \\ \hline
    \textbf{jackknife } & No guarantee           & $\approx 1-\alpha$ & $n$                 & $n+n_{\text{test}}$         \\ \hline
    \textbf{jackknife+} & $\geq 1-2\alpha$       & $\approx 1-\alpha$ & $n$                 & $n \cdot n_{\text{test}}$         \\ \hline
    \end{tabular}}
    \label{table: summary_pi}
\end{table}

Clearly the cost of model training is the largest cost, especially when training a complicated model such as a neural network. Our method attempts to only train a single model (like the na\"ive model) while still providing a coverage guarantee (like jackknife+).

\section{Method and Algorithm}

Our method involves using the lazy estimation framework~\citep{chizat2020lazy} with an initialization using differential privacy stochastic gradient descent (DP-SGD, \citet{Abadi_2016}). 

\subsection{Lazy Estimation}

To calculate the predictive interval fast and reduce the computational cost, we consider using a lazy estimation scheme when estimating the leave-one-out (LOO) models for calculating the predictive interval. The key idea of our approach is given an initialization parameterized by $\theta_0$, for each $i$, we approximate $\hat f_{-i}$ by linearizing around initial base model and training the linearized model using all but the $i^{\rm th}$ training sample. Thus instead of training $n$ neural networks explicitly, we train a single neural network followed by $n$ linear models, leading to a substantial computational saving. 

For any dataset $\cD = \{(X_i, Y_i)\}$ and a large model $f(x; ~\theta)$, where $\theta$ denotes the model parameters, we define a lazy estimation operator
$\lazy_{\lambda} (\theta; \cD): \Theta\times \cD \mapsto \Theta\subset \bbR^M$ with the ridge regression parameter $\lambda > 0$  
and initialization $\theta_0$, we estimate the LOO models by taking a linearization around $\theta_0$ and minimize the penalized training loss on data $\cD_{\rmj} = {\{(X_i, Y_i)\}}_{i\in [n]\backslash [j]}$, 
\begin{equation}
    \lazy_{\lambda} (\theta_0; \cD_{\rmj})=  
    \theta_0  + \underset{\Delta\in \bbR^M}{\argmin} \left\{\sum_{i \in [n]\backslash \{j\}} \cL
    \left(Y_i, f(X_i, \theta_0)+ \Delta^\top \nabla_{\theta} f(X_i; \theta)|_{\theta = \theta_0}\right) + \lambda \|\Delta\|^2
    \right\}.
    \label{eq: lazy estimate}
\end{equation}

\subsection{Differential Privacy (DP) initialization}

One of the important questions with lazy estimation is how to choose the initialization $\theta_0$. For this we use the concept of \emph{differential privacy}, which allows us to achieve coverage guarantees with a large reduction in computational costs. 

\begin{definition}[Differential Privacy~\citep{Dwork2009}]
\label{def:dp}
    A randomized mechanism $\cM: \mathcal{D} \mapsto \mathcal{R}$ with domain $\mathcal{D}$ and range $\mathcal{R}$ satisfies \textit{$(\epsilon, \delta)$-differential privacy} if for any two adjacent datasets $d, d' \subseteq \mathcal{D}$ and for any subset of outputs $S\subseteq \mathcal{R}$ it holds that 
    \begin{equation}
        P\left(\cM(d) \in S \right) \leq
        e^{\epsilon} P\left(
        \cM(d') \in S
        \right) + \delta.
    \end{equation}
\end{definition}

In essence, this differential privacy (DP) condition ensures that the output distribution does not change much when the input data has some small change, which makes it hard to distinguish between the input databases on the basis of the output. 
\citet{Abadi_2016} proposes \cref{alg: dp sgd}, which performs stochastic gradient descent with noisy gradients and showed that it achieves $(\epsilon,\delta)$-differential privacy, where noise is added to the gradients in the scale of $O(\frac{\sqrt{\log{\delta}}}{n\epsilon})$. Algorithms to ensure DP are always designed based on the sensitivity of the original algorithm.
More precisely, given an algorithm $\cA$ and a norm function $\|\cdot\|$ over the range of $\cA$, the sensitivity of $\cA$ is defined as 
\begin{align}\label{eq: sensitivity}
    s(\cA) = \max_{\substack{d(D, D')=1, \card(D)=n}} \|\cA(D) - \cA(D')\|
\end{align}
The Laplacian mechanism to construct a differentially private algorithm is as follows (see \citep{koufogiannis2015optimality, holohan2018bounded}): for an algorithm (or function) $\cA: \cD\mapsto \Theta\subset \bbR^M$, the random function $\cA^{\DP}(d) = \cA(d) + \xi, ~\forall d \in \cD$ satisfies $(\epsilon,0)$-differential privacy, where the elements $\xi_i$ follows a Laplacian distribution $\text{Lap}(0, s(\cA)/\epsilon)$:
$p(\xi) \propto e^{-\epsilon |\xi|/ s(\cA)}$.
We can set $\epsilon$ and $\delta$ as small constants (\eg $\epsilon = 0.01$ and $\delta = 10^{-3}$), and adjust the noise levels in different DP mechanisms.

\begin{algorithm}[ht]
    \caption{DP-SGD~\citep{Abadi_2016}}\label{alg: dp sgd}
    \KwParameters{Noise scale $\sigma$; learning rate $\eta_t$;
    group size $l$, gradient norm bound $C$; number of interactions $T$.}
    \KwData{$\cD =  \{Z_1, \dots, Z_n\}=\{(X_1, Y_1), \dots, (X_n, Y_n)\} $;} 
    \KwInit{$\theta_0$ randomly;}
    \For{$t\in [T]$}{
        Take a random sample $L_t$ with sampling probability $l/n$;\\
        For each $i\in L_t$,
        compute the gradient $g_t(Z_i) \leftarrow \nabla_{\theta_t} \cL(\theta_t; Z_i)$;\\
        Clip gradient: $\bar{g}_t(Z_i) = g_t(Z_i) / \max(1, \frac{\|g_t(Z_i)\|_2}{C})$;\\
        Add Noise: $\tilde{g}_t(Z_i) \leftarrow \frac{1}{|L_t|}(\sum_i \bar{g}_t(Z_i) + \cN(0, \sigma^2C^2 I))$;\\
        Descent: $\theta_{t+1} \leftarrow \theta_t - \eta_t \tilde{g}_t$.
    }
    Output $\theta_T$ with the overall privacy cost $(\epsilon, \delta)$ using a privacy accounting method.
\end{algorithm}

\subsection{Our method: Lazy Estimation with Differential Privacy}

First, we use a noisy SGD method~\citep{Abadi_2016} with $(\epsilon, \delta)$ differential privacy to get a full model parameter estimation  $\widehat \theta^{\text{DP}}_n({\epsilon, \delta})$, which for convenience will be denoted as $\widehat \theta_n^{\text{DP}}$. Fig.~\ref{fig:method_outline} provides an overview of our procedure.

\begin{figure}
    \centering
    \label{fig:method_outline}
    \includegraphics[width=0.9\textwidth]{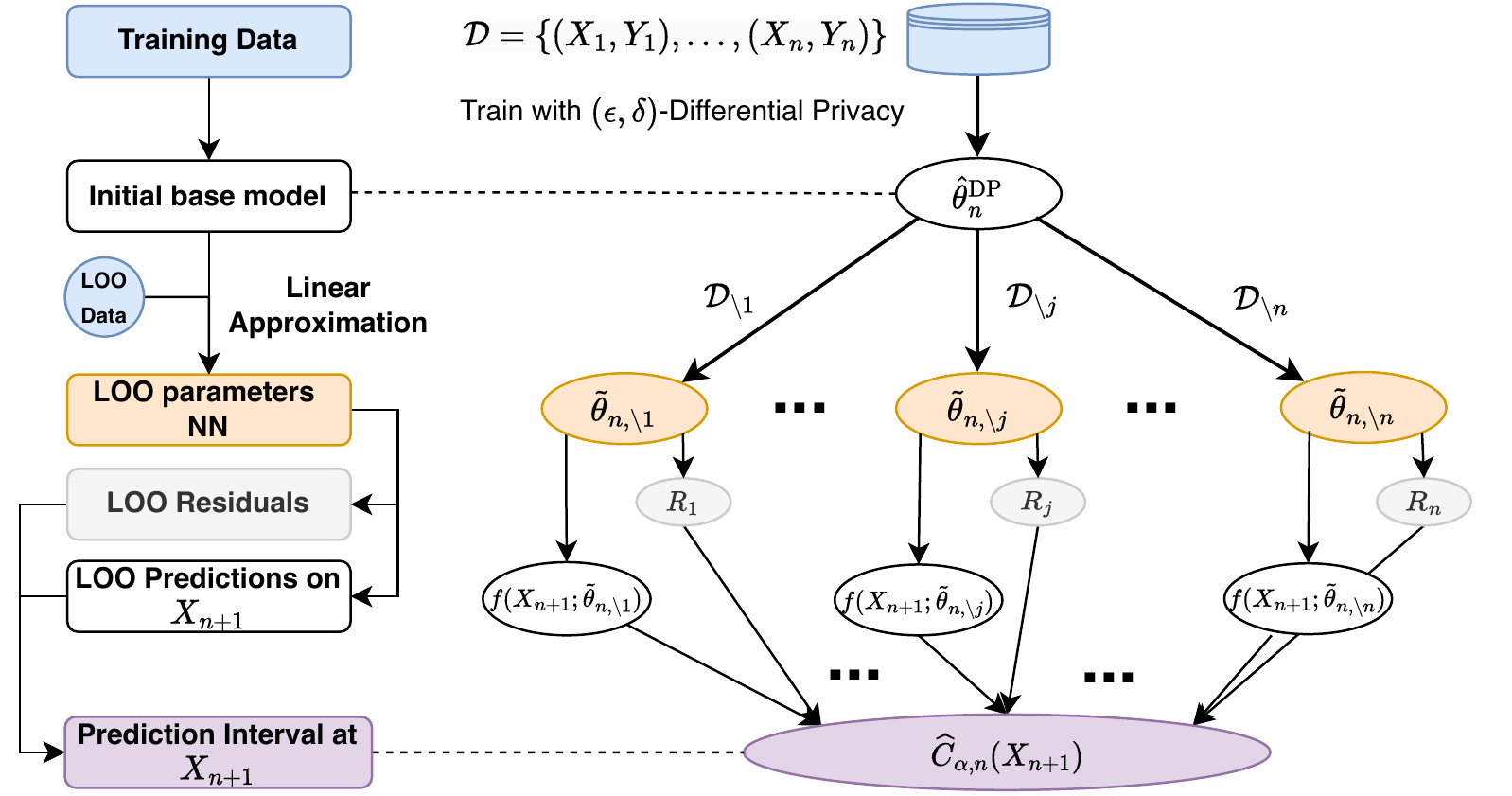}
    \caption{The outline of DP-Lazy PI method to calculate prediction intervals}
\end{figure}

Based on this $\widehat \theta_n^{\text{DP}}$, we estimate the LOO model parameters by using lazy training to take a linearization around $\widehat \theta_n^{\text{DP}}$ and minimize the penalized training loss on data $\cD_{\rmj} = {\{(X_i, Y_i)\}}_{i\in [n]\backslash [j]}$, 
\begin{equation}
    \widetilde{\theta}_{n, \rmj}= \lazy_{\lambda} (\widehat{\theta}^{\DP}_n; ~\cD_{\rmj}).
\end{equation}
Here $\lazy_{\lambda} (\cdot; \cdot)$ is defined in \eqref{eq: lazy estimate}.
Based on the LOO model parameters, we can calculate the LOO residuals by 
\[ R_j := |Y_j - f(X_j;~ \widetilde{\theta}_{n,\rmj})|.\]

For a test data $X_{n+1}$, its predictive interval using our method is:
\begin{equation}
    \label{eq: define_relaxed_ci}
    \widehat{C}_{\alpha, n}^{\nu} (X_{n+1}) \coloneqq
\left[Q_{\alpha}^{-}\left\{f(X_{n+1};~ \widetilde{\theta}_{n,\rmj}) 
- R_j \right\} - \nu,
~Q_{\alpha}^{+}\left\{f(X_{n+1};~ \widetilde{\theta}_{n,\rmj}) + 
 R_j\right\} + \nu \right]
\end{equation}
where $\nu$ is a relaxation term that is close to 0.
Our full method is presented in \cref{alg: lazy_pi_dp}.

 \begin{algorithm}[ht]
    \caption{DP-Lazy PI}\label{alg: lazy_pi_dp}
    \KwParameters{DP parameters ($\epsilon, \delta$); Ridge Parameter $\lambda$; Relaxation Parameter $\nu$;}
    \KwData{Training samples $\cD = \{(X_1, Y_1), \dots, (X_n, Y_n)\subseteq \cX\times \cY\}$; test sample $X_{n+1}$;} 
    \KwResult{Predictive Interval $\widehat{C}_{\alpha,n}^{\text{Lazy}}$;}
    Compute NN weights $\widehat\theta_n^{\DP}$: 
    \[ \widehat\theta_n^{\DP} \leftarrow \text{DP-SGD}(\cD; (\epsilon, \delta)).\]\\
    \For{$j\in 1, \dots, n$}{
        Compute $\widetilde\theta_{n, \rmj}$ by Lazy method defined in \eqref{eq: lazy estimate} on data $\cD \backslash \{(X_j, Y_j)\}$ initialized at $\widehat{\theta}_n^{\DP}$ 
        \[\widetilde\theta_{n, \rmj} \leftarrow \lazy_{\lambda}\left(\widehat{\theta}_{n}^{\DP}; \cD_{\rmj}\right). \]
         \\ 
        Compute the LOO residuals \[R_j \leftarrow |Y_j - f(X_j; \widetilde{\theta}_{n, \rmj})|.\]
    }
    Compute the DP-Lazy PI at $X_{n+1}$:
        \[\widehat{C}^{\nu}_{\alpha, n} (X_{n+1}) = 
        [Q_{\alpha}^{-}\{f(X_{n+1}; \widetilde{\theta}_{n,\rmj}) - R_j\} -\nu,
         ~Q_{\alpha}^{+}\{f(X_{n+1}; \widetilde{\theta}_{n,\rmj}) + R_j\} + \nu] .\]
 \end{algorithm}

\section{Coverage Guarantee}

Recall that $\widehat{\theta}^{\DP}_{n}$ is the 
NN parameter estimate from $(\epsilon, \delta)$-DP algorithm using the full training data 
$\cD = \{Z_1, \dots, Z_n\} = \{(X_1,Y_1), \dots, (X_n, Y_n)\}$.
For any $j\in [n]$, define $\widehat \theta_{\rmj}^{\DP}$ as 
NN parameter estimate from an $(\epsilon, \delta)$-DP algorithm $\cA^{\DP}$ 
 trained with the $j$-th sample deleted, \ie using data
$\cD_{\rmj} = \cD \backslash \{Z_j\}$.
Based on $\widehat \theta_{\rmj}^{\DP}$, we define another lazy estimate:
\begin{equation}
    \label{eq: exchangeable approx}
    \widetilde{\theta}_{\rmj, \rmj}= 
    \widehat{\theta}^{\DP}_{\rmj}  + \arg\min_{\Delta \theta} \left\{\sum_{i \in [n]\backslash \{j\}} \mathcal{L}
    \left(Y_i, f(X_i, \widehat{\theta}^{\DP}_{\rmj})+ \Delta \theta^\top \nabla_{\theta} f(X_i; \theta)|_{\theta = \widehat{\theta}^{\DP}_{\rmj}}\right) + \lambda \|\Delta \theta\|^2
    \right\}.
\end{equation}

\begin{theorem}\label{thm: main_result}
    For any $\nu>0$ and $\eta>0$, such that
    \begin{equation}\label{eq: stable condition}
        \bbP\left(|f(X_{n+1};~ \widetilde{\theta}_{n,\rmj}) -
        f(X_{n+1};~ \widetilde{\theta}_{\rmj,\rmj})| > \nu/2\right) \leq \eta, 
    \end{equation}
    the coverage of the DP-Lazy prediction interval 
    $\widehat{C}_{\alpha, n}^{\nu} (x)$ 
    defined in \cref{eq: define_relaxed_ci} is 
    larger than $1-2\alpha - 3\sqrt{2\eta + 2\epsilon
    + \delta}$, where $(\epsilon, \delta)$ are the DP parameters, \ie
    \begin{equation}
        \bbP\left[ Y_{n+1}\in \widehat{C}_{\alpha, n}^{\nu} (X_{n+1})\right] \geq 
        1-2\alpha - 3\sqrt{2\eta + 2\epsilon + \delta}.
    \end{equation}
\end{theorem}
\begin{remark}
    In \cref{thm: main_result}, essentially \eqref{eq: stable condition} requires an \emph{out-of-sample} stability in the DP algorithm, which appears elsewhere in the literature \cite{foygel2019predictive} and is referred to as ``hypothesis stability'' in \cite{bousquetStabilityGeneralization}. We want to emphasize that this out-of-sample stability only requires that, for a test data point that is \emph{independent of the training data}, the predicted value does not change much if we remove one point in the training data set, which can hold even for algorithms that suffer from strong overfitting in contrast to the \emph{in-sample} stability. As an illustrative example, the K-nearest-neighbor algorithm is shown to satisfy the out-of-sample stability with $\nu=0$ and $\eta=K/n$ (see Example 5.5 in \cite{foygel2019predictive}).

    For neural networks, notions of stability are also assumed in the literature \citep{verma2019stability, forti1994necessary}.
    In the Supplementary Material, we also give an example showing that if $x$ is multivariate Gaussian, for a two-layer neural network (one hidden layer) with activation functions like ReLU, sigmoid or tanh and $s(\cA) = O(1/n)$ (defined in \cref{eq: sensitivity}),  \eqref{eq: stable condition} holds true with $\nu = O(1/n)$ up to some log terms and $\eta = O(e^{-\epsilon n})$ for some constant $c$.

\end{remark}

\subsection{Proof Overview of \texorpdfstring{\cref{thm: main_result}}{}}
In this section, we'll provide the key ideas and the skeleton of the proof for the coverage guarantee, while the completed proof is postponed in the Appendix.
Recall
 $R_j := |Y_j - f(X_j;~ \widetilde{\theta}_{n,\rmj})|$
 and define
  $R_{j,\rmj} := |Y_j - f(X_j;~ \widetilde{\theta}_{\rmj,\rmj})|.$

First, for any $\alpha<\alpha'<1$ we show that a jackknife+ type interval $\widetilde{C}_{\alpha', n} (X_{n+1})$ defined as 
\begin{equation}
\widetilde{C}_{\alpha', n} (x) \coloneqq
\left[Q_{\alpha'}^{-}\left\{f(x;~ \widetilde{\theta}_{\rmj,\rmj}) - 
R_{j,\rmj}\right\} ,
~Q_{\alpha'}^{+}\left\{f(x;~ \widetilde{\theta}_{\rmj,\rmj}) + 
R_{j,\rmj}\right\} \right]
\end{equation}
is contained within our rapidly calculated prediction interval $\widehat{C}^{\nu}_{\alpha, n} (X_{n+1})$ with high probability that depends on the relaxation term $\eta$ and the DP parameters $(\epsilon, \delta)$: 
{
\allowdisplaybreaks
\begin{align}
    \bbP&\left[
        \widetilde{C}_{\alpha', n} (X_{n+1}) \nsubseteq 
        \widehat{C}_{\alpha, n}^{\nu} (X_{n+1})
    \right]\\
    &= \bbP\Bigg[
    \left\{Q_{\alpha}^{+}\left\{f(X_{n+1};~ \widetilde{\theta}_{n,\rmj}) + 
    R_j + \nu\right\}
    < Q_{\alpha'}^{+}\left\{f(X_{n+1};~ \widetilde{\theta}_{\rmj,\rmj}) + 
    R_{j,\rmj}\right\}\right\}\label{eq: def}
    \\
    &\qquad \qquad 
        \cup\left\{Q_{\alpha}^{-}\left\{f(X_{n+1};~ \widetilde{\theta}_{n,\rmj}) - 
        R_j - \nu\right\}
        > Q_{\alpha'}^{-}\left\{f(X_{n+1};~ \widetilde{\theta}_{\rmj,\rmj}) - 
        R_{j,\rmj}\right\}\right\}
        \Bigg]\nonumber\\
    & \leq  \bbP \Bigg[
        \Bigg\{ \sum_{j=1}^n \bbone\Big(
        \left|f(X_{n+1};~ \widetilde{\theta}_{\rmj,\rmj})
        -f(X_{n+1};~ \widetilde{\theta}_{n,\rmj})\right|\nonumber\\
        & \qquad \qquad + \left|f(X_j;~ \widetilde{\theta}_{\rmj,\rmj})
         - f(X_j;~ \widetilde{\theta}_{n,\rmj})\right| 
         > \nu
    \Big) \geq (\alpha'-\alpha)(n+1)\Bigg\}
    \Bigg] \label{eq: event count}\\
    & \leq \frac{1}
    {(\alpha'-\alpha)}
    \Bigg\{
        \bbP\left[
        \left|f(X_{n+1};~ \widetilde{\theta}_{\rmj,\rmj})
        -f(X_{n+1};~ \widetilde{\theta}_{n,\rmj})\right|
         > \nu/2
    \right]\nonumber\\
    & \qquad \qquad + 
    \bbP\left[
        \left|f(X_j;~ \widetilde{\theta}_{\rmj,\rmj})
         - f(X_j;~ \widetilde{\theta}_{n,\rmj})\right| > \nu/2
    \right] 
    \Bigg\}\label{eq: sum split}\\
    & \leq \frac{1}
    {(\alpha'-\alpha)}
    \left\{
        2\bbP\left[
        \left|f(X_{n+1};~ \widetilde{\theta}_{\rmj,\rmj})
        -f(X_{n+1};~ \widetilde{\theta}_{n,\rmj})\right|
         > \nu/2
    \right]
    +
    2\epsilon+\delta
    \right\} \label{eq: by dp} \\
    & \leq  \frac{2\eta + 2\epsilon + \delta}{\alpha'-\alpha} \label{eq: by stability}.
\end{align}
}

The probabilities are taken with respect to all the training data $\cD$ as well as the test data $(X_{n+1}, Y_{n+1})$. \eqref{eq: def} holds by the definitions of the prediction intervals of interest, which is equivalent to \eqref{eq: event count}; by data exchangeability and Markov inequality, \eqref{eq: sum split} holds true; 
by the property of differential privacy in the Appendix, the in-sample stability term  in \eqref{eq: sum split} is relaxed to an out-of-sample stability condition in \eqref{eq: by dp},
which is bounded by the condition in \cref{eq: stable condition}.

By the jackknife+ coverage guarantee in \cite{foygel2019predictive} that 
\begin{equation}
    \bbP(Y_{n+1} \notin  \widetilde{C}_{\alpha', n} (X_{n+1})) \leq 2\alpha',
\end{equation}
we can bound the miscoverage rate for the prediction interval 
$\widehat{C}_{\alpha, n}^{\nu} (X_{n+1})$:
\begin{align}
    \bbP(Y_{n+1} \notin & \widehat{C}_{\alpha, n}^{\nu} (X_{n+1}))\\
    \leq & \bbP(Y_{n+1} \notin  \widetilde{C}_{\alpha', n} (X_{n+1})) + \bbP\left[
        \widetilde{C}_{\alpha', n} (X_{n+1}) \nsubseteq 
        \widehat{C}_{\alpha, n}^{\nu} (X_{n+1})
    \right]\\
    \leq & 2\alpha' + \frac{2\eta + 2\epsilon + \delta}{\alpha' - \alpha}
\end{align}
for all $\alpha'>\alpha$. Take $\alpha' = \alpha + \sqrt{2\eta + 2\epsilon + \delta}$, we therefore have the coverage of $\widehat{C}_{\alpha, n}^{\nu} (X_{n+1})$: 
\begin{align*}
    \bbP(Y_{n+1} \in & \widehat{C}_{\alpha, n}^{\nu} (X_{n+1})) 
    = 1- \bbP(Y_{n+1} \notin  \widehat{C}_{\alpha, n}^{\nu} (X_{n+1}))\\
    \geq & 1-2\alpha - 3\sqrt{2\eta + 2\epsilon + \delta}.
\end{align*}

\section{Experiments}
We compare the following methods for estimating prediction intervals:
(1) jackknife+: defined in \cref{sec:jackknife+}, with the base algorithm being neural networks with random initialization;
(2) DP-Lazy PI (labeled \textit{lazy\_dp} in plots): our proposed method with the same NN architecture used in jackknife+, with the relaxation term $\nu$ set to zero, and differential parameters set as $\epsilon=0.01, ~\delta=10^{-3}$;
(3) Lazy PI without DP (labeled \textit{lazy\_finetune} in plots): removes the privacy mechanism compared to DP-Lazy PI, \ie use $\widehat{\theta}_{n}$ instead of $\widehat{\theta}_{n}^{\DP}$ as $\theta_0$ in the lazy estimation step \eqref{eq: lazy estimate}.

For evaluation, we consider the following three performance aspects on the test data set $\cD_{\test}$ with size $n_{\test}$ for a prediction interval $\widehat{C}_n(\cdot)$: (1) Coverage: $\frac{1}{n_{\test}}\sum_{(x,y)\in \cD_{\test}} \bbone\left[y\in \widehat{C}_n(x)\right]$(2) Compute time; (3) Average interval width: $\frac{1}{n_{\test}}\sum_{(x,y)\in \cD_{\test}}  \text{len}[\widehat{C}_n(x)]$.
As we set the level as $\alpha=0.1$, our target of the coverage is $1-2\alpha = 80\%$, we want to emphasize that a higher coverage of the prediction intervals is not always desirable, since it may suggest that the prediction intervals are overly wide and conservative.

\subsection{Simulation}\label{sec:simulation}
To generate the data, $\{X_i, ~i\in [N]\}$ are randomly selected from a Gaussian distribution 
$X_i\in \bbR^{p} \sim \cN(0, 5\cdot\mathbb{I}_{p}),~ i\in N$, where $N=5000$.  The responses $Y_i, ~i\in [N]$ are generated by:
\begin{equation}
    Y_i = \sqrt{\ReLU(X_i^\top \beta)} + \epsilon_i, \text{ where } \epsilon_i \sim \cN(0, 0.5), ~
    \beta \sim \text{Beta(1.0, 2.5)}.
\end{equation}

We consider a neural network with two hidden layers, each contains 64 nodes. 
The data are randomly split into the training data 
with the training size $n = 100$, and evaluation set with the size $n_{\test} = 4900$.

We construct prediction intervals using jackknife+, lazy finetune and DP-Lazy PI respectively on the training data with $\alpha=0.1$. When training the neural networks, the training batch size is $10$ and the max number of epochs is $10$. The penalty parameter is taken as $\lambda =10$ in the lazy type of methods.

As shown in \cref{fig: sim_normal}, the coverage of our method is closer to the target $80\%$ coverage and in terms of the compute time, lazy methods reduce the average compute time significantly compared to jackknife+. As the feature dimension increases, the coverage is still guaranteed; we can also see in \cref{fig: sim_normal} (b) that compared to lazy finetune without noise added in the full model training, lazy DP becomes better in the interval width, which might come from the interpolation (near zero training error) in the high-dimensional case that makes the LOO estimates in lazy finetune stays at the initial base model in the lazy procedure.

\begin{figure}
    \centering
    \begin{tabular}{{@{}c@{}}}
        \includegraphics[width=0.95\textwidth]{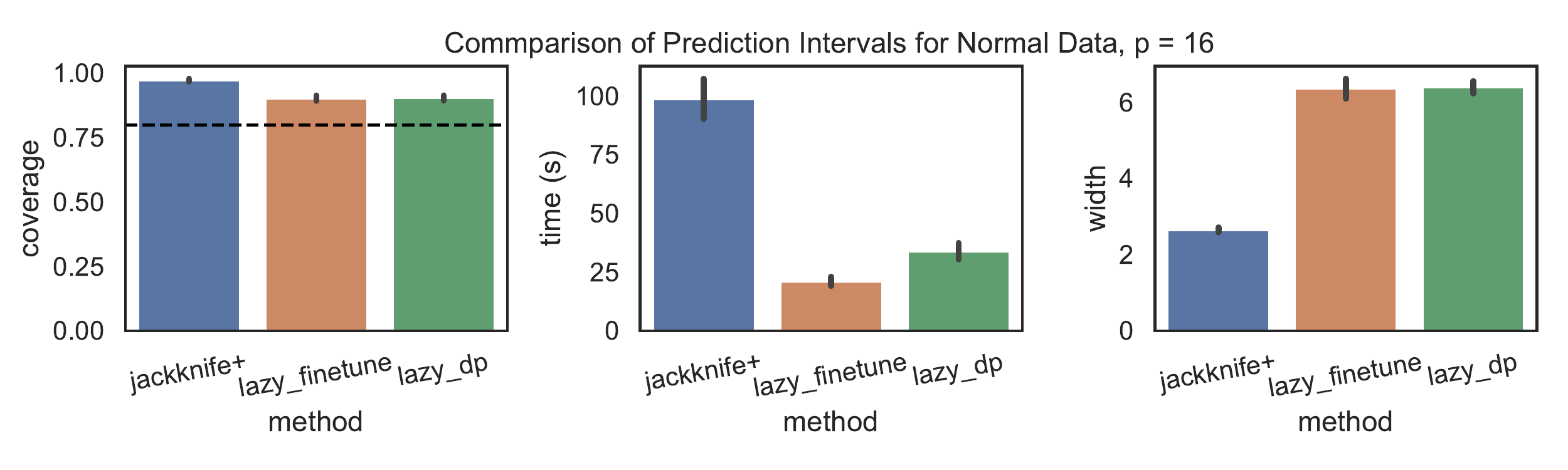}\\
        \small{(a) Simulated data with dimension $p=16$}
    \end{tabular}
    \begin{tabular}{{@{}c@{}}}
    \includegraphics[width = 0.95\textwidth]{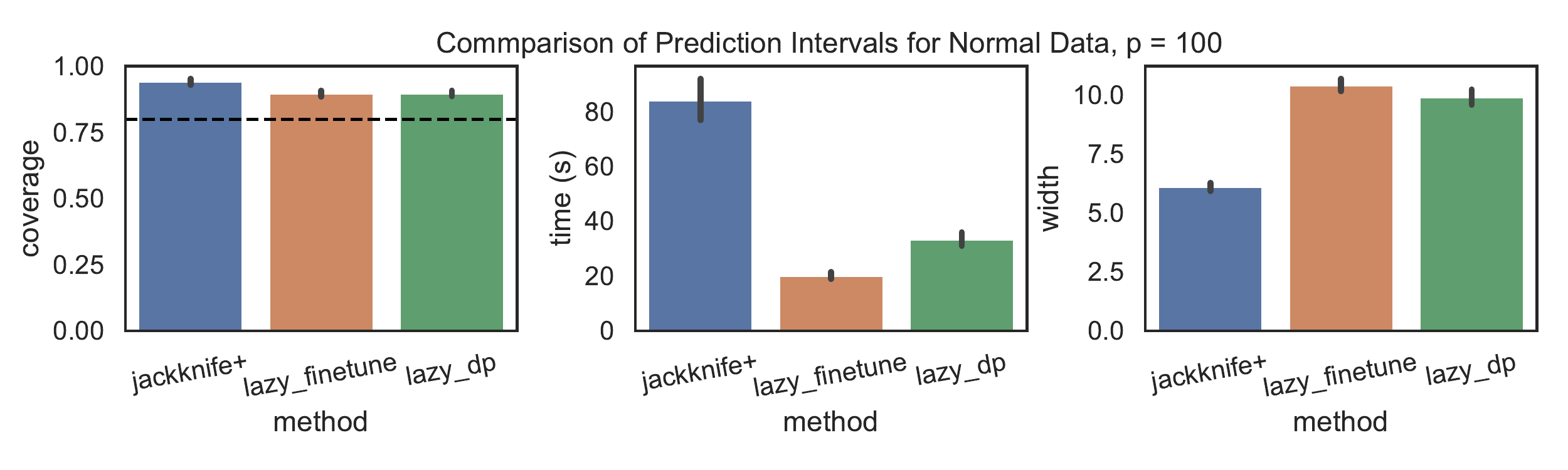}\\
    \small{(b) Simulated data with dimension $p=100$}
    \end{tabular}
\caption{The average coverage, average computing time and width of the prediction intervals with $\alpha=0.1$ in 15 trials, with the error bars to show the $\pm$ standard errors. The dotted line in the coverage panel indicates the target coverage level $80\%$.
}\label{fig: sim_normal}
\end{figure}

\subsection{Real Data}
\paragraph{Data sets}
(1) \emph{The BlogFeedbac}k\footnote{\url{https://archive.ics.uci.edu/ml/datasets/BlogFeedback}} data set \citep{spiliopoulou2014data,foygel2019predictive}, 
contains information on 52397 blog posts with $p=280$ covariates. 
The response is the number of comments left on the blog post in the following 24 hours, which we transformed as $Y = \log(1 + \#\text{comments})$. 
(2) \emph{The Medical Expenditure Panel Survey 2016 data set}\footnote{\url{https://meps.ahrq.gov/mepsweb/data_stats/download_data_files_detail.jsp?cboPufNumber=HC-192}} contains $33005$ records on individuals' utilization of medical services such as visits to the doctor, hospital stays etc with feature dimension $p=107$ with relevant features such as age, race/ethnicity, family income, occupation type, etc.  Our goal is the predict the health care system utilization of each individual, which is a composite score reflecting the number of visits to a doctor's office, hospital visits, days in nursing home care, etc. 

\paragraph{Experiment Settings}
The training size to construct the prediction interval is set $n = 100$. We consider the 3-layer neural network with hidden layers $(64, 64)$ as the base NN architecture.
To train the model parameters in jackknife+, we initialize the NN randomly for each LOO models. 
The penalty level is set as $\lambda = 10$, and the DP parameters are $\epsilon=0.01, ~\delta=10^{-3}$. The batch size is $10$ and the maximum number of epochs is 10.
We repeat the trials $15$ times with different random seeds for the train-test split and the random initialization of NN parameters in jackknife+. The figures show the average coverage, computing time and the interval width across these random trails.

\begin{figure}[ht]
    \centering
    \begin{tabular}{{@{}c@{}}}
        \includegraphics[width=0.95\textwidth]{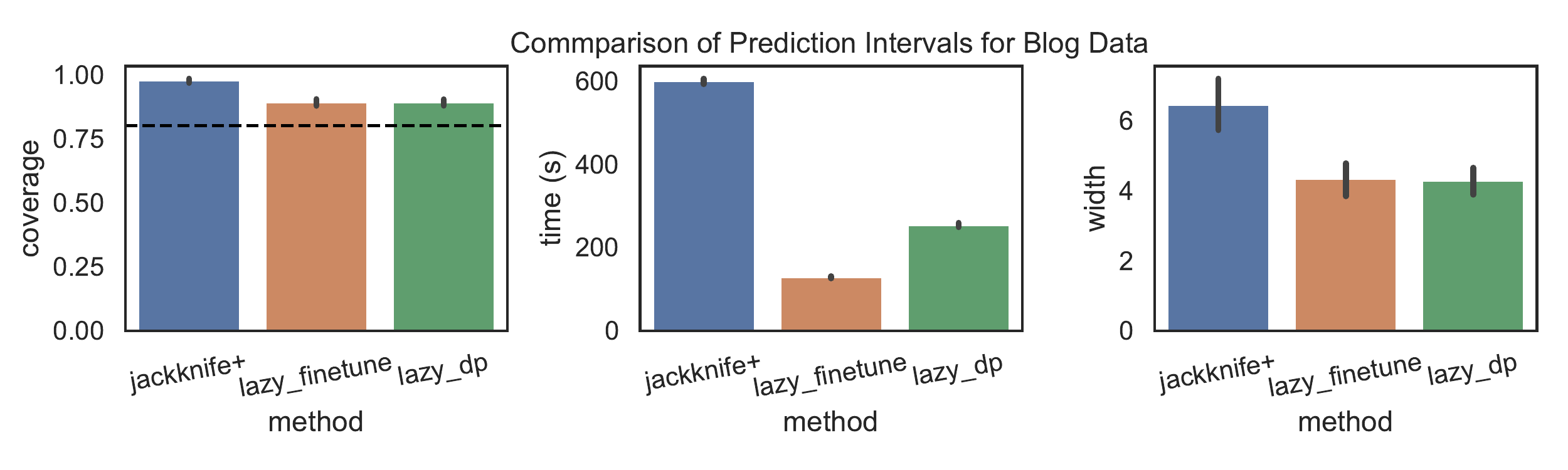}\\
        \small{(a) BlogFeedback data set}
    \end{tabular}
    \begin{tabular}{{@{}c@{}}}
    \includegraphics[width = 0.95\textwidth]{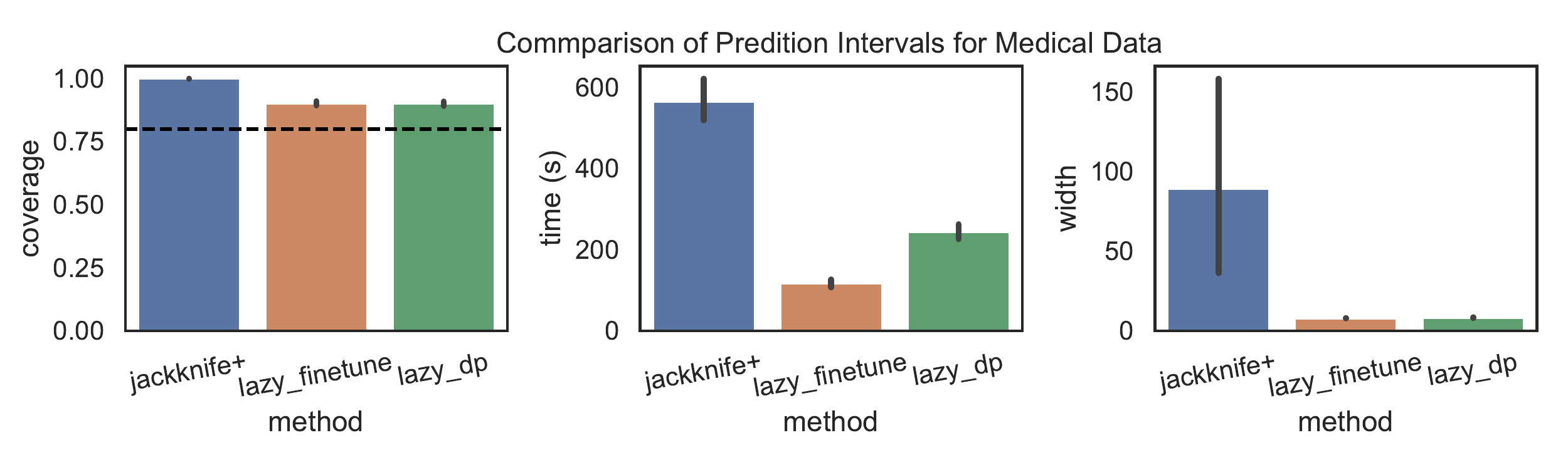}\\
    \small{(b) Medical Expenditure Panel Survey 2016 data set}
    \end{tabular}
    \caption{The coverage, average computing time and width of the prediction intervals on real data sets with $\alpha = 0.1$. The dotted line in the left panels in (a) and (b) is the target $80\%$ coverage. 
    }\label{fig: real data}
\end{figure}

\paragraph{Results}
As shown in \cref{fig: real data} (a) and (b), when we look into the real data sets, DP-Lazy PI decreases the computing time while enjoys a narrower interval width, while achieving the coverage guarantee at $80\%$. Compared to the data generated in the simulation, for real data with much complicated data structure, the model mis-specification makes the width of the jackknife+ estimator greater and the implicit regularization of the lazy approach reduces this width.
The implementation codes can be found in \url{https://github.com/VioyueG/DP_LAZY_PI}.
\section{Discussion and Limitations}

In this paper, we describe a new method, Differential Privacy Lazy Predictive Inference or DP-Lazy PI, which provides a fast method for distribution-free inference for neural networks. Our method involves using differentially private stochastic gradient descent (DP-SGD) as an initial estimate and then computing leave-one-out approximations using lazy training and then re-combining the leave-one-out estimators to form the predictive interval. Importantly, we are able to provide coverage guarantees that closely match the coverage guarantees provided for the jackknife+ procedure with a fraction of the computational cost since only a single neural network model needs to be trained.

The two real data examples also suggest that when we have significant model
misspecification, the implicit regularization that lazy training provides means that our DP Lazy-PI method has narrower width prediction intervals than jackknife+. 

An alternative method we evaluated in our experiments was lazy finetune, which removes the privacy mechanism prior to lazy training. The simulation results show very similar and at times slightly improved performance compared to our DP Lazy-PI approach; however, a limitation is that we lack theoretical guarantees for this approach. It remains an open question whether we can provide theoretical guarantees for lazy finetune, noting that in the zero training error regime where we can perfectly interpolate, lazy training does not apply since we are at a zero gradient initialization, making linearized models useless. However, in higher dimensions, the implicit regularization through early stopping of neural network training may be the reason lazy training works well. A further limitation of this work is that our method requires the base algorithm (\eg the neural network estimation) to be stable and satisfy the condition 
\eqref{eq: stable condition}.

\newpage

\section*{Acknowledgements}
    R. Willett gratefully acknowledges the support of AFOSR grant FA9550-18-1-0166 and NSF grants DMS-2023109 and DMS-1925101. G. Raskutti acknowledges the support of NIH grant R01 GM131381-03.

\bibliographystyle{plainnat}
\bibliography{references}
\appendix
\section{Appendix}

\subsection{Complete Proof of Theorem 1}

Recall the leave-one-out residuals we used in DP-Lazy PI
 \[ R_j := |Y_j - f(X_j;~ \widetilde{\theta}_{n,\rmj})|\]
 and define a set of jackknife+ type residuals
  \[ R_{j,\rmj} := |Y_j - f(X_j;~ \widetilde{\theta}_{\rmj,\rmj})|.\]
 
We construct a jackknife+ type of predictive interval as a reference interval for theoretical purposes:
\begin{equation}
    \widetilde{C}_{\alpha, n} (x) \coloneqq
\left[Q_{\alpha}^{-}\left\{f(x;~ \widetilde{\theta}_{\rmj,\rmj}) - 
R_{j,\rmj}\right\} ,
~Q_{\alpha}^{+}\left\{f(x;~ \widetilde{\theta}_{\rmj,\rmj}) + 
R_{j,\rmj}\right\} \right].
\end{equation}

\begin{lemma}\citep{foygel2019predictive}
\label{lemma: ref_coverage}
    For any test data $(X_{n+1}, Y_{n+1})$, the coverage of the $\widetilde{C}_{\alpha, n} (X_{n+1})$ is:
   \begin{equation}
    \bbP(Y_{n+1} \in \widetilde{C}_{\alpha, n} (X_{n+1})) \geq 1- 2\alpha, ~\forall \alpha\in (0, 1/2) 
   \end{equation}
\end{lemma}

\Cref{lemma: ref_coverage} gives the coverage of the jackknife+ interval $\widetilde{C}_{\alpha, n} (X_{n+1})$. 
This predictive interval is only constructed for theoretical purposes as a reference interval, while in practice, we avoid it due to its prohibitive computational cost. We show that the probability that our rapidly calculated predictive interval $\widehat{C}^{\nu}_{\alpha, n} (X_{n+1})$ does not contain a reference interval $\widetilde{C}_{\alpha', n} (X_{n+1})$ can be bounded by a small term that related to $\eta$ and the DP parameters $(\epsilon, \delta)$.
The proof of \Cref{lemma: ref_coverage} is based on the results in \cite{foygel2019predictive}.

\subsubsection*{Complete Proof of Theorem 1}
\begin{proof} 
In the following equations, we intend to bound the event that the right endpoint of the interval $\widetilde{C}_{\alpha', n} (x)$ is larger than the right endpoint of the interval $\widehat{C}^{\nu}_{\alpha, n} (x)$ for any $0<\alpha< \alpha'$:
\begin{align}
    &\left\{
    Q_{\alpha}^{+}\left\{f(x;~ \widetilde{\theta}_{n,\rmj}) + 
    R_j + \nu\right\}
    < Q_{\alpha'}^{+}\left\{f(x;~ \widetilde{\theta}_{\rmj,\rmj}) + 
    R_{j,\rmj}\right\}
    \right\}\label{appendix_eq: right_end}\\
    & \qquad\subseteq \left\{ \sum_{j=1}^n \bbone\left(
        f(x;~ \widetilde{\theta}_{n,\rmj}) + 
    R_j + \nu 
    < f(x;~ \widetilde{\theta}_{\rmj,\rmj}) + 
    R_{j,\rmj}
    \right) \geq (\alpha'-\alpha)(n+1)\right\}\label{appendix_eq: count}\\
    &\qquad =\left\{\sum_{j=1}^n \bbone\left(
        f(x;~ \widetilde{\theta}_{\rmj,\rmj})
        -f(x;~ \widetilde{\theta}_{n,\rmj}) 
        +R_{j,\rmj}
         - R_j 
         > \nu
    \right) \geq (\alpha'-\alpha)(n+1)\right\}\nonumber\\
    &\qquad \subseteq  \Bigg\{\sum_{j=1}^n \bbone\left(
        \left|f(x;~ \widetilde{\theta}_{\rmj,\rmj})-f(x;~ \widetilde{\theta}_{n,\rmj})\right|
        +\left|f(X_j;~ \widetilde{\theta}_{\rmj,\rmj})
         - f(X_j;~ \widetilde{\theta}_{n,\rmj})\right| 
         > \nu
    \right) \nonumber \\
    & \qquad \qquad \qquad \geq (\alpha'-\alpha)(n+1)\Bigg\}.\label{appendix_eq: right_end_bound}
\end{align}
The first inclusion \revise{from \eqref{appendix_eq: right_end} to} \eqref{appendix_eq: count} holds true because the $\lceil (1-\alpha')(n+1) \rceil$-th largest value among the set ${\{f(x;~ \widetilde{\theta}_{\rmj,\rmj}) + 
R_{j,\rmj}\}}_{j\in [n]}$ is larger than the 
$\lceil (1-\alpha)(n+1) \rceil$-th largest value among the set ${\{f(x;~ \widetilde{\theta}_{n,\rmj}) + 
R_j + \nu\}}_{j\in [n]}$ by the definition of quantiles in \eqref{appendix_eq: right_end}.
By the fact that $\alpha < \alpha'$, we know that $\lceil (1-\alpha')(n+1) \rceil \leq \lceil (1-\alpha)(n+1) \rceil$. 
Therefore, for any $\lceil (1-\alpha')(n+1) \rceil \leq j \leq \lceil (1-\alpha)(n+1) \rceil $, we always have 
\begin{equation}\label{appendix_eq: explain_j}
    f(x;~ \widetilde{\theta}_{n,\rmj}) + 
    R_j + \nu 
    < f(x;~ \widetilde{\theta}_{\rmj,\rmj}) + 
    R_{j,\rmj},
\end{equation}
\ie  there exist at least $\lceil(\alpha' - \alpha)(n+1)\rceil$ number of $j$-s, such that \eqref{appendix_eq: explain_j} holds true.

The symmetric event of \eqref{appendix_eq: right_end} is that the left endpoint of the interval $\widetilde{C}_{\alpha', n} (x)$ is smaller than the left endpoint of the interval $\widehat{C}^{\nu}_{\alpha, n} (x)$ for any $0<\alpha< \alpha'$. We can bound this event with the same technique:
\begin{align}
    &\left\{
    Q_{\alpha}^{-}\left\{f(x;~ \widetilde{\theta}_{n,\rmj}) 
    - R_j - \nu\right\}
    > Q_{\alpha'}^{-}\left\{f(x;~ \widetilde{\theta}_{\rmj,\rmj}) 
    -R_{j,\rmj}\right\}
    \right\}\label{appendix_eq: left_end}\\
    & \qquad \subseteq \left\{ \sum_{j=1}^n \bbone\left(
        f(x;~ \widetilde{\theta}_{n,\rmj})  
        -R_j - \nu 
    > f(x;~ \widetilde{\theta}_{\rmj,\rmj}) - 
    R_{j,\rmj}
    \right) \geq (\alpha'-\alpha)(n+1)\right\}\nonumber\nonumber\\
    & \qquad =\left\{\sum_{j=1}^n \bbone\left(
        f(x;~ \widetilde{\theta}_{n,\rmj}) 
        -f(x;~ \widetilde{\theta}_{\rmj,\rmj})
        +R_{j,\rmj}
         - R_j 
         > \nu
    \right) \geq (\alpha'-\alpha)(n+1)\right\}\nonumber\nonumber\\
    & \qquad \subseteq \Bigg\{\sum_{j=1}^n \bbone\left(
        \left|f(x;~ \widetilde{\theta}_{\rmj,\rmj})
        -f(x;~ \widetilde{\theta}_{n,\rmj})\right|
        +\left|f(X_j;~ \widetilde{\theta}_{\rmj,\rmj})
         - f(X_j;~ \widetilde{\theta}_{n,\rmj})\right| 
         > \nu
    \right) \nonumber \\
    & \qquad \qquad \qquad \geq (\alpha'-\alpha)(n+1)\Bigg\}.
    \label{appendix_eq: left_end_bound}
\end{align}

The event that $\widetilde{C}_{\alpha', n} (x) \nsubseteq 
\widehat{C}_{\alpha, n}^{\nu} (x)$ happens when either \eqref{appendix_eq: right_end} or \eqref{appendix_eq: left_end} happens, while from the above inclusions, these two events are both included in the event in \eqref{appendix_eq: right_end_bound}.

Hence for any $\alpha' > \alpha$, 
\begin{align}
    \bbP&\left[
        \widetilde{C}_{\alpha', n} (X_{n+1}) \nsubseteq 
        \widehat{C}_{\alpha, n}^{\nu} (X_{n+1})
    \right]\\
    &= \bbP\Bigg[
    \left\{Q_{\alpha}^{+}\left\{f(X_{n+1};~ \widetilde{\theta}_{n,\rmj}) + 
    R_j + \nu\right\}
    < Q_{\alpha'}^{+}\left\{f(X_{n+1};~ \widetilde{\theta}_{\rmj,\rmj}) + 
    R_{j,\rmj}\right\}\right\}
    \\
    &\qquad \qquad 
        \cup\left\{Q_{\alpha}^{-}\left\{f(X_{n+1};~ \widetilde{\theta}_{n,\rmj}) - 
        R_j - \nu\right\}
        > Q_{\alpha'}^{-}\left\{f(X_{n+1};~ \widetilde{\theta}_{\rmj,\rmj}) - 
        R_{j,\rmj}\right\}\right\}
        \Bigg]\nonumber\\
    & \leq  \bbP \Bigg[
        \Bigg\{ \sum_{j=1}^n \bbone\Big(
        \left|f(X_{n+1};~ \widetilde{\theta}_{\rmj,\rmj})
        -f(X_{n+1};~ \widetilde{\theta}_{n,\rmj})\right|\nonumber\\
        & \qquad \qquad + \left|f(X_j;~ \widetilde{\theta}_{\rmj,\rmj})
         - f(X_j;~ \widetilde{\theta}_{n,\rmj})\right| 
         > \nu
    \Big) \geq (\alpha'-\alpha)(n+1)\Bigg\}
    \Bigg] \label{appendix_eq: union_event}\\
    & \leq \frac{1}{(\alpha'-\alpha)(n+1)}\bbE\Bigg[
        \sum_{j=1}^n \bbone\Big(
        \left|f(X_{n+1};~ \widetilde{\theta}_{\rmj,\rmj})
        -f(X_{n+1};~ \widetilde{\theta}_{n,\rmj})\right| 
        \nonumber  \label{appendix_eq: Markov}
        \\
        & \qquad \qquad 
        +\left|f(X_j;~ \widetilde{\theta}_{\rmj,\rmj})
         - f(X_j;~ \widetilde{\theta}_{n,\rmj})\right| 
         > \nu
    \Big)\Bigg] \\
    & = \frac{n\bbP\left[
        \left|f(X_{n+1};~ \widetilde{\theta}_{\rmj,\rmj})
        -f(X_{n+1};~ \widetilde{\theta}_{n,\rmj})\right|
        +\left|f(X_j;~ \widetilde{\theta}_{\rmj,\rmj})
         - f(X_j;~ \widetilde{\theta}_{n,\rmj})\right| 
         > \nu
    \right]}
    {(\alpha'-\alpha)(n+1)}  \label{appendix_eq: sum exch}\\
    & \leq \frac{1}
    {(\alpha'-\alpha)}
    \Bigg\{
        \bbP\left[
        \left|f(X_{n+1};~ \widetilde{\theta}_{\rmj,\rmj})
        -f(X_{n+1};~ \widetilde{\theta}_{n,\rmj})\right|
         > \nu/2
    \right]\nonumber\\
    & \qquad \qquad + 
    \bbP\left[
        \left|f(X_j;~ \widetilde{\theta}_{\rmj,\rmj})
         - f(X_j;~ \widetilde{\theta}_{n,\rmj})\right| > \nu/2
    \right] 
    \Bigg\}\label{appendix_eq: sum split}
    \\
    & \leq \frac{1}
    {(\alpha'-\alpha)}
    \left\{
        2\bbP\left[
        \left|f(X_{n+1};~ \widetilde{\theta}_{\rmj,\rmj})
        -f(X_{n+1};~ \widetilde{\theta}_{n,\rmj})\right|
         > \nu/2
    \right]
    +
    2\epsilon+\delta
    \right\} \label{appendix_eq: by dp} \\
    & \leq  \frac{2\eta + 2\epsilon + \delta}{\alpha'-\alpha}.\label{appendix_eq: assump}
\end{align}
The above set of equations/inequalities aim to bound the event that a $\nu$-relaxed lazy-DP interval
$\widehat{C}_{\alpha, n}^{\nu} (X_{n+1})$
doesn't contain a jackknife+ interval
$\widetilde{C}_{\alpha', n} (X_{n+1})$ based on the data exchangeability and the $(\epsilon, \delta)$ differential privacy in the algorithm. 
\revise{The probabilities are taken with respect to all the training data $\cD$ as well as the test data $(X_{n+1}, Y_{n+1})$.}
Specifically, \eqref{appendix_eq: union_event} comes from the inclusions in \eqref{appendix_eq: right_end_bound},\eqref{appendix_eq: left_end_bound}; \eqref{appendix_eq: Markov} comes from the Markov equality that $\bbP(|X|\geq t) \leq \bbE(|X|)/t$ for any random variable $X$ and non-negative constant $t$; as we calculate the sum of expectations, the exchangeability of the training data makes \eqref{appendix_eq: sum exch} hold true; by the fact that 
\begin{align*}
    &\left\{\left|f(X_{n+1};~ \widetilde{\theta}_{\rmj,\rmj})
        -f(X_{n+1};~ \widetilde{\theta}_{n,\rmj})\right|
        +\left|f(X_j;~ \widetilde{\theta}_{\rmj,\rmj})
         - f(X_j;~ \widetilde{\theta}_{n,\rmj})\right| 
         > \nu\right\}\\
    & \qquad \subseteq  
    \left\{
        \left|f(X_{n+1};~ \widetilde{\theta}_{\rmj,\rmj})
        -f(X_{n+1};~ \widetilde{\theta}_{n,\rmj})\right|
         > \frac{\nu}{2}
    \right\}  \cup
    \left\{
        \left|f(X_j;~ \widetilde{\theta}_{\rmj,\rmj})
         - f(X_j;~ \widetilde{\theta}_{n,\rmj})\right| > \frac{\nu}{2}
    \right\},
\end{align*}
we could split \eqref{appendix_eq: sum exch} into the form \eqref{appendix_eq: sum split}.
{we have a $(n+1)$ factor in the denominator of \eqref{appendix_eq: sum exch} thus the $n$ factor can be dropped by $n/(n+1) \leq 1$.
}
\eqref{appendix_eq: by dp} is derived by the differential privacy property in \cref{lemma: dp_conditional_ind}; the final bound in \eqref{appendix_eq: assump} holds true due to the out-of-sample stability.

If we take $\alpha' = \alpha + \sqrt{2\eta + 2\epsilon + \delta}$, by the fact that 
\begin{equation}
    \bbP(Y_{n+1} \notin  \widetilde{C}_{\alpha', n} (X_{n+1})) \leq 2\alpha',
\end{equation}
we can bound the miscoverage rate for the predictive interval 
$\widehat{C}_{\alpha, n}^{\nu} (X_{n+1})$:
\begin{align}
    \bbP(Y_{n+1} \notin & \widehat{C}_{\alpha, n}^{\nu} (X_{n+1}))\\
    \leq & \bbP(Y_{n+1} \notin  \widetilde{C}_{\alpha', n} (X_{n+1})) + \bbP\left[
        \widetilde{C}_{\alpha', n} (X_{n+1}) \nsubseteq 
        \widehat{C}_{\alpha, n}^{\nu} (X_{n+1})
    \right]\\
    \leq & 2\alpha' + \sqrt{2\eta + 2\epsilon + \delta}\\
    = & 2\alpha + 3\sqrt{2\eta + 2\epsilon + \delta}.
\end{align}
Hence the coverage of $\widehat{C}_{\alpha, n}^{\nu} (X_{n+1})$ is 
\begin{align*}
    \bbP(Y_{n+1} \in & \widehat{C}_{\alpha, n}^{\nu} (X_{n+1})) \\
    =& 1- \bbP(Y_{n+1} \notin  \widehat{C}_{\alpha, n}^{\nu} (X_{n+1}))\\
    \geq & 1-2\alpha - 3\sqrt{2\eta + 2\epsilon + \delta}.
\end{align*}
\end{proof}

\subsection{Proofs of Supporting Lemmas}
For any two random variables $U, V$, and $(U, V)$ has the joint distribution $p(u,v)$;
then the \emph{marginal distribution} of $U$ is $p_{U}(u) = \int_{v} p(u, v) dv$, and similarly $p_V(v) = \int_{u} p(u, v) du$; 
$(U\otimes V)$ has the 
the product distribution of $U$ and $V$, \ie 
\[
    p_{(U\otimes V)} (u, v) = p_{U}(u)\times p_V(v).
    \]

Therefore, the distribution of $(U\otimes V)$ is equivalent to the joint distribution of $(U, V')$, where $V'$ is an \emph{independent copy} of $V$, \ie $V'$ has the same marginal distribution as $V$, but $V'$ is independent with $U$ and $V$.

To describe the ``similarity'' of two random vectors, we introduce the following notion of indistinguishability that is also used in \cite{Kasiviswanathan_2014,rogers2016maxinformation}.
\begin{definition} (Indistinguishability \citep{Kasiviswanathan_2014})
    Two random vectors $V_1, V_2$ in a space $\Omega$ are $(\epsilon, \delta)$-indistinguishable, denoted 
    $ V_1\approx_{\epsilon, \delta} V_2$,
    if for all $S\in \Omega$, we have 
    \begin{align*}
        &\bbP[V_1\in S] \leq e^{\epsilon} \cdot \bbP[V_2\in S] + \delta  ~~\text{and}\\
        &\bbP[V_2\in S] \leq e^{\epsilon} \cdot \bbP[V_1\in S] + \delta. 
    \end{align*}
\end{definition}
    
The following lemma is utilized in the course of proof for Lemma 3.2 in \citet{rogers2016maxinformation}. The proof is provided as well for the completeness of the paper.
    
\begin{lemma}
\label{lemma: max info}
    Let $\cA^{\DP}: {(\cX\times \cY)}^n \mapsto \Theta$ be a $(\epsilon, \delta)$-DP algorithm,and $\widehat\theta^{\DP}_n = \cA^{\DP}(\cD)$ is the pretrained estimator with the full training data $\cD = (Z_1,\dots ,Z_n), ~ Z_i = (X_i, Y_i)$. Denote $\cD_{\rmj} = \cD \backslash \{Z_j\} \subseteq {(\cX\times \cY)}^{n-1}$ as the leave-one-out data set with the $j$-th data deleted.
    Then for all $j\in [n]$,  we have
    \begin{equation}
        (\widehat\theta^{\DP}_n, Z_j) \given[\big] \cD_{-j} \approx_{\epsilon, \delta} (\widehat\theta^{\DP}_n  \otimes Z_j) \given[\big] \cD_{-j},
        \label{appendix_eq:lem2}
    \end{equation}
    \ie  for {an independent copy} $Z_j'$ of $Z_j$, we have 
    \begin{equation}
        \begin{split}
            &\bbP((\widehat\theta^{\DP}_n, Z_j)\in S \given[\big]\cD_{\rmj}) \leq e^{\epsilon} \bbP((\widehat\theta^{\DP}_n, Z_j')\in S \given[\big]\cD_{\rmj}) +\delta \\
            \text{ and } \qquad & \bbP((\widehat\theta^{\DP}_n, Z_j)\in S \given[\big]\cD_{\rmj}) \geq e^{-\epsilon} \bbP((\widehat\theta^{\DP}_n, Z_j')\in S \given[\big]\cD_{\rmj}) -\delta,
        \end{split}
        \label{appendix_eq:lem2b}
    \end{equation}
    for any $S\subseteq \Theta\times(\cX\times \cY)$.
\end{lemma}

\begin{proof}     
    
    Fix any set $S \subseteq \Theta \times (\cX\times \cY)$. We then define $S_{z_j} = \{\theta\in \Theta: (\theta, z_j) \in S \}$ for any realization of $Z_j=z_j$.
    We now have
    \begin{equation}
        \label{appendix_eq: dp max info}
        \begin{split}
            &\bbP[(\cA^{\DP}(\cD), Z_j)\in S \given[\big]\cD_{\rmj}] \\
             & \qquad = \sum_{z_j\in (\cX\times \cY)} \bbP[Z_j = z_j] \bbP[\cA^{\DP}(\cD) \in S_{z_j} \given[\big]\cD_{\rmj}, Z_j = z_j] \\ 
            & \qquad \leq \sum_{z_j \in (\cX\times \cY)} \bbP[Z_j = z_j] \{
                e^{\epsilon} \bbP[\cA^{\DP}(\cD)\in S_{z_j} \given[\big]\cD_{\rmj}, Z_j=z_j'] + \delta
            \}, 
        \end{split}
    \end{equation}
    for any realization $z_j'\in (\cX\times \cY)$ for $Z_j$. Here the last line in \eqref{appendix_eq: dp max info} comes from the fact that $\cA^{\DP}(\cdot)$ is $(\epsilon, \delta)$ differentially private.
    By multiplying $\sum_{z_j'\in (\cX\times \cY)} \bbP[Z_j = z_j'] = 1$ to the RHS and LHS in \eqref{appendix_eq: dp max info}, we have 
    {\allowdisplaybreaks
        \begin{align}
        &\bbP[(\cA^{\DP}(\cD), Z_j) \in S \given[\big]\cD_{\rmj} ] \\
            & \qquad= \sum_{z_j'\in (\cX\times \cY)} \bbP[Z_j = z_j'] \bbP[(\cA^{\DP}(\cD), Z_j) \in S \given[\big]\cD_{\rmj} ] \\
            & \qquad\leq  
            \sum_{z_j'\in (\cX\times \cY)} \bbP[Z_j = z_j'] \sum_{z_j\in (\cX\times \cY)} \bbP[Z_j = z_j]
            \left\{
                e^{\epsilon} \bbP[\cA^{\DP}(\cD)\in S_{z_j} \given[\big]\cD_{\rmj} , Z_j=z_j'] + \delta
            \right\}\label{Appendix_eq:by_dp2}\\
            & \qquad=
            \sum_{z_j\in (\cX\times \cY)} \bbP[Z_j = z_j] \sum_{z_j'\in (\cX\times \cY)} \bbP[Z_j = z_j']
            \left\{
                e^{\epsilon} \bbP[\cA^{\DP}(\cD)\in S_{z_j} \given[\big]\cD_{\rmj} , Z_j=z_j'] + \delta
            \right\}\\
            & \qquad=  \sum_{z_j\in (\cX\times \cY)} \bbP[Z_j = z_j] 
            \left\{ e^{\epsilon} \bbP[\cA^{\DP}(\cD)\in S_{z_j} \given[\big]\cD_{\rmj} ] +\delta
                \right\}\\
            & \qquad= e^{\epsilon} \sum_{z_j\in (\cX\times \cY)} \bbP[Z_j =z_j]\bbP[\cA^{\DP}(\cD)\in S_{z_j} \given[\big]\cD_{\rmj} ] + \delta\\
            & \qquad= e^{\epsilon} \bbP[\cA^{\DP}(\cD) \otimes Z_j \in S \given[\big]\cD_{\rmj}  ] + \delta.
    \end{align}}
    With a similar argument, we can prove its symmetric counterpart:
    \begin{equation}
        \bbP[\cA^{\DP}(\cD) \otimes Z_j \in S \given[\big]\cD_{\rmj} ] \leq e^{\epsilon} \bbP[(\cA^{\DP}(\cD), Z_j) \in S \given[\big]\cD_{\rmj} ] + \delta.
    \end{equation}
    Since $\widehat\theta_n^{\DP}= \cA^{\DP}(\cD)$, we finish the proof that $\widehat\theta_n^{\DP}\otimes Z_j$ is ($\epsilon, \delta$)-indistinguishable with $(\widehat\theta_n^{\DP}, Z_j)$ given the other data points $\cD_{\rmj} $, \ie  
    \begin{equation*}
        (\widehat\theta_n^{\DP}, Z_j)\given[\big]\cD_{\rmj}  \approx_{\epsilon, \delta} (\widehat\theta_n^{\DP}\otimes Z_j)\given[\big]\cD_{\rmj}.
    \end{equation*}
\end{proof}

\begin{lemma}\label{col: dependency}
    The dependency of the leave-one-out estimator $\widetilde{\theta}_{n, \rmj}$ on the $j$-th data point $Z_j = (X_j, Y_j)$ can be controlled by
    \begin{equation}
        (\widetilde\theta_{n, \rmj}, Z_j)\given[\big] \cD_{\rmj} \approx_{\epsilon, \delta}  (\widetilde\theta_{n, \rmj} \otimes Z_j)\given[\big] \cD_{\rmj},
    \end{equation}
    \ie for an independent copy of $Z_j'$ for $Z_j$,
    \begin{equation}
        \begin{split}
            &\bbP((\widetilde\theta_{n, \rmj}, Z_j)\in S \given[\big]\cD_{\rmj}) \leq e^{\epsilon} \bbP((\widetilde\theta_{n, \rmj}, Z_j')\in S \given[\big]\cD_{\rmj}) +\delta \qquad \text{and}\\
            & \bbP((\widetilde\theta_{n, \rmj}, Z_j)\in S \given[\big]\cD_{\rmj}) \geq e^{-\epsilon} \bbP((\widetilde\theta_{n, \rmj}, Z_j')\in S \given[\big]\cD_{\rmj}) -\delta.
        \end{split}
    \end{equation}
\end{lemma}
\begin{proof}
    Recall the definition of $\widetilde\theta_{n, \rmj}$:
    \begin{equation*}
        \begin{split}
            \widetilde{\theta}_{n, \rmj}
            =
    \widehat{\theta}^{\DP}_n  +\arg\min_{\Delta \theta} \left\{\sum_{i \in [n]\backslash \{j\}} \mathcal{L}
    \left(Y_i, f(X_i, \widehat{\theta}^{\DP}_n)+ \Delta \theta^\top \nabla_{\theta} f(X_i; \theta)|_{\theta = \widehat{\theta}^{\DP}_n}\right) + \lambda \|\Delta \theta\|^2
    \right\}.
        \end{split}
    \end{equation*}
    
    Denote $\widehat{G}(\bX_{\rmj}; \theta)\in \bbR^{(n-1)\times (n-1)} := \nabla_{\theta} f(\bX_{\rmj}; ~\theta) \cdot {\nabla_{\theta} f(\bX_{\rmj}; ~\theta)}^\top$ as the gram matrix (\aka~the NTK matrix) on data $\cD_{\rmj} = (\bX_{\rmj}, \bY_{\rmj})$. Define $g(\cdot| \cD_{\rmj}): \Theta \mapsto \Theta$:
    \begin{equation*}
        \begin{split}
           g(\theta \given[\big] \cD_{\rmj})
            \coloneqq  \theta  + {\nabla_{\theta} f(\bX_{\rmj}; ~\theta)}^\top {\left(
                \widehat{G}(\bX_{\rmj}; \theta) + \lambda I_{n-1}
            \right)}^{-1} \cdot 
            \left[
                \bY_{\rmj} - f(\bX_{\rmj}; ~ \theta)
            \right];
        \end{split}
    \end{equation*}
    then the explicit form of $\widetilde{\theta}_{n, \rmj}$ can be represented as:
    \begin{equation}
        \widetilde{\theta}_{n, \rmj}\given[\big]  \bz_{\rmj}= g(\widehat \theta_n^{\DP} | \cD_{\rmj} = \bz_{\rmj}) 
            = g(\cM(\cD)| \cD_{\rmj}).
    \end{equation}
    For any $\cS\subseteq \Theta\times (\cX\times \cY)$, $\cS_{g, \cD_{\rmj}} \coloneqq \{(\theta, z): \forall z\in (\cX\times \cY), (g(\theta|\cD_{\rmj}),z)\in \cS \}$. Therefore 
    \begin{subequations}
    \begin{align}
             &\bbP\left[(\widetilde{\theta}_{n, \rmj}, Z_j)\in \cS \given[\big] \cD_{\rmj} \right]\\
             & \qquad=\sum_{z_j \in (\cX \times \cY)} \bbP(Z_j = z_j)\bbP\left[(g(\widehat{\theta}_{n}^{\DP}\given[\big] \cD_{\rmj}), Z_j)\in \cS| \cD_{\rmj} , Z_j = z_j\right]\\
             & \qquad=\bbP\left[
                 (\widehat{\theta}_{n}^{\DP}, Z_j)\in \cS_{g, \cD_{\rmj}} \given[\big] \cD_{\rmj}
             \right]\\
             & \qquad\leq e^{\epsilon} \bbP\left[
                \widehat{\theta}_{n}^{\DP}\otimes Z_j\in \cS_{g, \cD_{\rmj}} \given[\big] \cD_{\rmj}
            \right] + \delta \label{Appendix_eq:DP_max_info}\\
            & \qquad=e^{\epsilon} \bbP\left[
                \widetilde{\theta}_{n, \rmj}\otimes Z_j \in \cS \given[\big] \cD_{\rmj} 
            \right] +\delta.
    \end{align}
    \end{subequations}
    where \eqref{Appendix_eq:DP_max_info} is derived from \cref{lemma: max info}.
    Similarly, we can show that 
    \[\bbP\left[(\widetilde{\theta}_{n, \rmj}, Z_j)\in \cS \given[\big] \cD_{\rmj} \right] \geq e^{-\epsilon}  \bbP\left[
        \widetilde{\theta}_{n, \rmj}\otimes Z_j \in \cS \given[\big] \cD_{\rmj} 
    \right] -\delta,\]
     finishing the proof of \cref{col: dependency}.
\end{proof}

\begin{lemma}(Closeness of in-sample and out-of-sample stability)\label{lemma: dp_conditional_ind}
    For any $x\in \bbR^{+}$, we have
    \begin{align}
        & \Bigg|\bbP\left[
            \left|f(X_{n+1};~ \widetilde{\theta}_{\rmj,\rmj})
        -f(X_{n+1};~ \widetilde{\theta}_{n,\rmj})\right|
         > x
    \right] \\
    & \qquad \qquad
    -
    \bbP\left[
        \left|f(X_j;~ \widetilde{\theta}_{\rmj,\rmj})
         - f(X_j;~ \widetilde{\theta}_{n,\rmj})\right| > x
    \right]\Bigg| \leq 2\epsilon + \delta.
    \end{align}
\end{lemma}
\begin{remark}
    As we can see, $\left|f(X_{n+1};~ \widetilde{\theta}_{\rmj,\rmj})
-f(X_{n+1};~ \widetilde{\theta}_{n,\rmj})\right|$ illustrates the change of \emph{out-of-sample} prediction when a training data $Z_j$ is removed in the model training, while $\left|f(X_{j};~ \widetilde{\theta}_{\rmj,\rmj})
-f(X_{j};~ \widetilde{\theta}_{n,\rmj})\right|$ depicts the change of \emph{in-sample} prediction when a training data $Z_j$ is removed. Normally the in-sample and out-of-sample stability are different and have to be considered separately as in \citet{foygel2019predictive}, but due to the differential privacy in the training algorithm, we can show through the following lemma that these two types of stabilities can be bridged by differential privacy, therefore we only need out-of-sample stability in the algorithm.
\end{remark}

\begin{proof}
    First, we know by the definition of $\widetilde{\theta}_{\rmj,\rmj}$ and $\widetilde{\theta}_{n,\rmj}$ and the fact that $\cD_{\rmj} \ind \{(X_j, Y_j)\}$, we know that $\widetilde{\theta}_{\rmj,\rmj} \ind (\widetilde{\theta}_{n,\rmj}, X_j) \given[\Big] \cD_{\rmj}$; therefore, for any $S_1, S_2\subseteq \Theta$ and $S_3 \subseteq \cX$, we have
    \begin{align}
        &\bbP\left[
            (\widetilde{\theta}_{\rmj,\rmj}, \widetilde{\theta}_{n,\rmj}, X_j) \in S_1 \times S_2 \times S_3 \given[\Big] \cD_{\rmj}
        \right] \\
        & \qquad=\bbP\left[\widetilde{\theta}_{\rmj,\rmj} \in S_1\given[\Big] \cD_{\rmj}\right] \cdot
        \bbP\left[
            (\widetilde{\theta}_{n,\rmj}, X_j) \in S_2 \times S_3 \given[\Big] \cD_{\rmj}
        \right]\label{appendix_eq: lemma8_l1}\\
        & \qquad\leq \bbP\left[\widetilde{\theta}_{\rmj,\rmj} \in S_1\given[\Big] \cD_{\rmj}\right] \left\{
            e^{\epsilon} \bbP\left[
                (\widetilde{\theta}_{n,\rmj}, X_{n+1}) \in S_2 \times S_3 \given[\Big] \cD_{\rmj}
            \right] + \delta
        \right\}\label{appendix_eq: lemma8_l2}\\
        & \qquad\leq e^{\epsilon}\bbP\left[\widetilde{\theta}_{\rmj,\rmj} \in S_1\given[\Big] \cD_{\rmj}\right]\bbP\left[
            (\widetilde{\theta}_{n,\rmj}, X_{n+1}) \in S_2 \times S_3 \given[\Big] \cD_{\rmj} 
        \right]+\delta\\
        & \qquad=e^{\epsilon}
        \bbP\left[
            (\widetilde{\theta}_{\rmj,\rmj}, \widetilde{\theta}_{n,\rmj}, X_{n+1}) \in S_1 \times S_2 \times S_3 \given[\Big] \cD_{\rmj}
        \right]
        +\delta\label{appendix_eq: lemma8_l4}.
    \end{align}
    Here \eqref{appendix_eq: lemma8_l1} and \eqref{appendix_eq: lemma8_l4} holds true due to the conditional independence; \eqref{appendix_eq: lemma8_l2} is from \cref{col: dependency}.
    By the same argument, we have 
    \begin{align}
        &\bbP\left[
            (\widetilde{\theta}_{\rmj,\rmj}, \widetilde{\theta}_{n,\rmj}, X_j) \in S_1 \times S_2 \times S_3 \given[\Big] \cD_{\rmj}
        \right] \\
        &\qquad \qquad\geq 
        e^{-\epsilon}
        \bbP\left[
            (\widetilde{\theta}_{\rmj,\rmj}, \widetilde{\theta}_{n,\rmj}, X_{n+1}) \in S_1 \times S_2 \times S_3 \given[\Big] \cD_{\rmj}
        \right]
        -\delta.
    \end{align}
    Hence for any $x\in \bbR^+$, 
    \begin{subequations}
    \begin{align}
        &\Bigg|\bbP\left[
            \left|f(X_{n+1};~ \widetilde{\theta}_{\rmj,\rmj})
        -f(X_{n+1};~ \widetilde{\theta}_{n,\rmj})\right|
         > x
        \right]\\
        & \qquad \qquad \qquad
        -
        \bbP\left[
        \left|f(X_j;~ \widetilde{\theta}_{\rmj,\rmj})
         - f(X_j;~ \widetilde{\theta}_{n,\rmj})\right| > x
    \right]\Bigg| \\
    & \qquad \leq  \bbE \sup_{\substack{S_1, S_2\subseteq \Theta\\ S_3\in \cX}}\Bigg|
    \bbP\left[
        ((\widetilde{\theta}_{\rmj,\rmj}, \widetilde{\theta}_{n,\rmj}, X_{n+1}) \in S_1 \times S_2 \times S_3) \given[\Big] \cD_{\rmj}
    \right] \\
    & \qquad \qquad \qquad -
    \bbP\left[
        ((\widetilde{\theta}_{\rmj,\rmj}, \widetilde{\theta}_{n,\rmj}, X_j) \in S_1 \times S_2 \times S_3) \given[\Big] \cD_{\rmj}
    \right]
    \Bigg|\nonumber\\
    & \qquad \leq \bbE \sup_{\substack{S_1, S_2\subseteq \Theta\\ S_3\in \cX}}\left|
    (e^{\epsilon}-1)\bbP\left[
        ((\widetilde{\theta}_{\rmj,\rmj}, \widetilde{\theta}_{n,\rmj}, X_{n+1}) \in S_1 \times S_2 \times S_3) \given[\Big] \cD_{\rmj}
    \right]+\delta
    \right|\\
    & \qquad \leq  e^{\epsilon} - 1 + \delta \leq 2\epsilon + \delta,
    \end{align}
    \end{subequations}
    for any $\epsilon<1/2$.
\end{proof}

\section{Appendix. Discussion of the Stability Condition}

\subsection{The Laplacian Mechanism}\label{sec:laplacian}

Given an algorithm $\cA$ and a norm function $\|\cdot\|$ over the range of $\cA$, the sensitivity of $\cA$ is defined as 
\begin{equation} \label{append_eq: sensitivity}
    s(\cA) = \max_{\substack{d(D, D')=1,\\ \card(D)=n}} \|\cA(D) - \cA(D')\|.
\end{equation}

(Usually the norm function $\|\cdot \|$ is either $\ell_1$ or $\ell_2$ norm.)

Based on the definition of sensitivity defined as \cref{append_eq: sensitivity}, the Laplacian mechanism to construct differential private algorithm is as follows \citep{koufogiannis2015optimality, holohan2018bounded}: for an algorithm (or function) $\cA: \cD\mapsto \Theta\subset \bbR^M$, the random function $\cA^{\DP}(d) = \cA(d) + \xi, ~\forall d \in \cD$ satisfies $(\epsilon,0)$-differential privacy, where $\xi\in \bbR^{M}$ follows a Laplacian distribution $\text{Lap}(0, s(\cA)/\epsilon)$:
\[
p(\xi) \propto e^{-\epsilon \|\xi\|/ s(\cA)}.
\]

Specifically, let $\cA$ be the algorithm to find the empirical risk minimization(ERM) on a certain data set $\cD$:
\[\cA(\cD) = \arg\min_{\theta} \sum_{(X_i, Y_i)\in \cD} [Y_i - f(X_i; \theta)]^2;\]
Let $s(\cA)$ be it sensitivity defined in \cref{append_eq: sensitivity}. 

Let $\widehat{\theta}_{n} = \cA(\cD)$ and $\widehat{\theta}_{\rmj} = \cA(\cD\backslash \{Z_j\})$ be the ERM estimations respectively on the full training data $\cD$ and leave-one-out data $\cD_{\rmj}$;

Consider the $(\epsilon, 0)$-DP algorithm $\cA^{\DP}(\cdot) = \cA(\cdot) + \xi, ~ \xi \sim \text{Lap}(0, s(\cA)/\epsilon)$, and we have the $(\epsilon, 0)$-DP estimators
\[\widehat{\theta}_{n}^{\DP} = \widehat{\theta}_{n} + \xi;~~
\widehat{\theta}_{\rmj}^{\DP} = \widehat{\theta}_{\rmj} + \xi'.
\]

\begin{assumption}\label{asst:diff}
    The parametrization of neural network $f(\cdot; \theta)$ is differentiable with a locally Lipschitz differential $\nabla_{\theta} f$.
\end{assumption}

\begin{assumption} (Local strong convexity of neural networks)\label{asst:lip}
    For all $i \in [n]$, $f(\bX_{-j};\theta)$ is locally Lipschitz continuous w.r.t $\theta$ with Lipschitz constant $\Lip(f)$, and $\nabla_{\theta} f(\bX_{-j};\theta)$ is Lipschitz continuous with a Lipschitz constant $\Lip(\nabla f)$, \ie~ for any  $\theta_1, \theta_2\in \cB_{r}(\widehat{\theta}_{-j})$ with $r = 2(s(\cA)+ \sqrt{s(\cA)})$:
    \begin{align*} 
        &\| f(\bX_{-j};\theta_1) - f(\bX_{-j};\theta_2)\|_n \leq \Lip(f) \|\theta_1 - \theta_2\|;\\ & \| \nabla_{\theta} f(\bX_{-j};\theta_1) - \nabla_{\theta} f(\bX_{-j};\theta_2)\|_n \leq \Lip(\nabla f) \|\theta_1 - \theta_2\|,
    \end{align*}
    where $\|\cdot \|_{n}$ is defined as the empirical norm.
\end{assumption}
According to \cref{lemma:local_convexity}, for one-hidden-layer neural networks with activation function being ReLU, sigmoid or tanh, \label{asst:lip} can be satisfied. We'll revisit this assumption in \cref{sec:local_convexity}.

\begin{theorem}
    Consider the Laplacian mechanism with $(\epsilon, 0)$-DP parameters. Suppose \cref{asst:diff,asst:lip} and $s(\cA) = O(1/n)$, with probability at least $1- e^{-c\epsilon \sqrt{n}}$ for some constant $c$, 
    we have 
    \begin{equation}
        |f(x;~ \widetilde{\theta}_{n,\rmj}) -
    f(x;~ \widetilde{\theta}_{\rmj,\rmj})| = O\left(\frac{\log^2(1/n)}{n} \right).
    \end{equation}
\end{theorem}

\begin{proof}

\textbf{Step 1. Show that $\widehat\theta_{n}^{\DP}$ and $\widehat\theta_{\rmj}^{\DP}$ are both in the neighborhood of $\widehat{\theta}_{\rmj}$ with high probability.}

Recall the fact that $\widehat{\theta}_n^{\DP} = \widehat{\theta}_n + \xi$, by triangle inequality and the fact that $\|\widehat{\theta}_n-\widehat{\theta}_{\rmj}\| = \|\cA(\cD) -\cA(\cD_{\rmj})\| \leq s(\cA)$, we have
\begin{align*}
    \|\widehat{\theta}_n^{\DP} -\widehat{\theta}_{\rmj}\| 
    \leq \|\widehat{\theta}_n-\widehat{\theta}_{\rmj}\| + 
    \|\widehat{\theta}_n - \widehat{\theta}_n^{\DP}\|\leq s(\cA) + \|\xi\|;
\end{align*}
Therefore, by the property of $\text{Lap}(0, \frac{s(\cA)}{\epsilon})$ that $P(\|\xi\|>t) \leq \exp[-\frac{\epsilon t}{s(\cA)}]$,  we have
\begin{align*}
    &P\left(\|\widehat{\theta}_n^{\DP} -\widehat{\theta}_{\rmj}\| > s(\cA) + \sqrt{s(\cA)}\right)\\
    \leq &P\left(
        s(\cA) + \|\xi\| >s(\cA) + \sqrt{s(\cA)}
    \right)\\
    \leq & \exp\left(- \frac{\epsilon}{\sqrt{s(\cA)}}
    \right);
\end{align*}

On the other hand, $\widehat{\theta}_{\rmj}^{\DP}$ is obtained by adding perturbations $\xi'$ around the $\widehat{\theta}_{\rmj}$,
therefore we have
\begin{align*}
    P\left(\|\widehat{\theta}_{\rmj}^{\DP} -\widehat{\theta}_{\rmj}\| > \sqrt{s(\cA)}\right)
    \leq P\left(
         \|\xi'\| >\sqrt{s(\cA)}
    \right)
    \leq  \exp\left(- \frac{\epsilon}{\sqrt{s(\cA)}}
    \right).
\end{align*}

Hence with probability at least $1-2e^{-\frac{\epsilon}{\sqrt{s(\cA)}}}$, we have $\|\widehat{\theta}_{\rmj}^{\DP} -\widehat{\theta}_{\rmj}\| \leq \sqrt{s(\cA)}\leq  \sqrt{s(\cA)} + s(\cA)$ and $\|\widehat{\theta}_{n}^{\DP} -\widehat{\theta}_{\rmj}\| \leq  \sqrt{s(\cA)} + s(\cA)$.

\textbf{Step 2. Show that $\widetilde{\theta}_{n,\rmj}$ and $\widetilde{\theta}_{\rmj,\rmj}$ are close.}

Recall that $\widetilde{\theta}_{n,\rmj} = g_{\rmj}(\widehat{\theta}_{n}^{\DP})$ and 
$\widetilde{\theta}_{\rmj,\rmj} = g_{\rmj}(\widehat{\theta}_{\rmj}^{\DP})$,
denote $\nabla f(\theta_0) = \nabla_{\theta} f(\bX_{-j}; \theta_0) \in \bbR^{(n-1)\times M}$,
and the linearized function of $f(x; \theta)$ around $\theta_0$ as
\[
    \bar{f}_{\theta_0}(x; \theta) = f(x; \theta_0) + (\theta - \theta_0)^\top \nabla_{\theta} f(x; \theta_0)\]
then $g_{\rmj}(\theta_0)$ is the ERM of the linearized function $\bar{f}_{\theta_0}(x; \theta)$ on the LOO data set:
\begin{equation}
    \begin{split}
        &g_{\rmj}(\theta_0) = \theta_0 +  {\nabla f(\theta_0)}^\top {\left(
            \nabla f(\theta_0) {\nabla f(\theta_0)}^\top + \lambda I_{n-1}
        \right)}^{-1} 
        \left[
            \bY_{-j} - f(\bX_{-j};\theta_0)
        \right]\\
        &= \arg\min_{\theta \in \bbR^M}\left\{ \sum_{i\neq j} 
        \left[Y_i- f(X_i; \theta_0) - (\theta - \theta_0)^\top \nabla_{\theta} f(X_i; \theta_0)\right]^2
         + \lambda \| \theta - \theta_0 \|^2\right\}\\
        &\coloneqq \arg\min_{\theta \in \bbR^M} {\cL_{-j}}(\bar{f}_{\theta_0}(x; \theta)) +  \lambda \| \theta - \theta_0 \|^2;
    \end{split}
\end{equation}

Define the gradient flow of ${\cL_{-j}}(\bar{f}_{\theta_0}(x; \theta))$  in the parameter space $\bbR^M$ as $\bar{\theta}(t; \theta_0)$, which satisfies $\bar{\theta}(0; \theta_0) = \theta_0$ and solves the ODE:
\begin{equation}
    \bar{\theta}'(t;\theta_0) = -\nabla {\cL_{-j}}\left(\bar{f}_{\theta_0}(x; \bar{\theta}(t;\theta_0))\right)=
     \left[\bY_{-j} - \bar{f}_{\theta_0}(\bX_{-j}; \bar{\theta}(t;\theta_0))\right]^\top \nabla f(\theta_0);
\end{equation}

On the other hand, as $f(x; \theta) \approx f(x; \theta_0) + (\theta - \theta_0)^\top \nabla_{\theta} f(x; \theta_0)= \bar{f}_{\theta_0}(x; \theta)$, we define the gradient flow on the neural network $f(\cdot, \theta)$  initialized at $\theta_0$ (\ie ~${\theta}(t;\theta_0) =\theta_0$) and solves the ODE:
\begin{equation}
    {\theta}'(t;\theta_0) = -\nabla {\cL_{-j}}\left({f}_{\theta_0}(x; {\theta}(t;\theta_0))\right)=
     \left[\bY_{-j} - {f}(\bX_{-j}; {\theta}(t;\theta_0))\right]^\top \nabla f(\theta(t; \theta_0));
\end{equation}

By Theorem 2.3 in \cite{chizat2020lazy}, we can show that 
\begin{equation}
    \| {\theta}(t;\theta_0) - \bar{\theta}(t;\theta_0)\| \leq t^2 \Lip(f)^4 \Lip(\nabla f)  \left(
        \frac{2}{\Lip(f)} + \frac{4t}{3}\Lip(f)
    \right) \|\theta_0 - \widehat{\theta}_{-j}\|^2;
\end{equation}

By the strong convexity and continuity of $\cL_{-j}(\bar{f}_{\theta_0}(x; \theta))$, we can show that $\bar{\theta}(t;\theta_0)$ converge to its minimizer $g_{\rmj}(\theta_0)$  (take $\lambda = 0$) fast:
\begin{equation}
    \|\bar{\theta}(t; \theta_0) - g_{\rmj}(\theta_0)\| \leq (\mu_0/m_0)^2 \Lip(f) \exp(-m_0\lambda_{0} t)  \| \theta_0 - \widehat{\theta}_{-j}\|,
\end{equation}
where $\cL_{-j}(\bar{f}_{\theta_0}(x; \theta))$ is $\mu_0$ strong convex and $m_0$ Lipschitz continuous, and $\lambda_0>0$ is the smallest eigenvalue of $\nabla f(\theta_0)\nabla f(\theta_0)^\top$.

By Lemma B.1 in \cite{chizat2020lazy}, if $\cL_{-j}(f(x;\theta)): \bbR^M\mapsto \bbR$ is $m$ strong convex with $\mu$-Lipschitz continuous gradient and the smallest eigenvalue of $\nabla f(\theta(t))\nabla f(\theta(t))^\top$ lower bounded by $\lambda_{\min} >0$ for $0\leq t\leq T$, then $\theta(t; \theta_0)$ can converge fast to the minimizer $\widehat{\theta}_{-j}$:
\begin{equation}
    \|\theta(t; \theta_0) - \widehat{\theta}_{-j}\| \leq (\mu/m)^2 \Lip(f) \exp(-m\lambda_{\min} t)  \| \theta_0 - \widehat{\theta}_{-j}\|.
\end{equation}

Hence,
\begin{align}
         &\|g_{\rmj}(\theta_0) - \widehat{\theta}_{-j}\| \\
    \leq &\|g_{\rmj}(\theta_0) -\bar{\theta}(t; \theta_0)\| + \|\theta(t; \theta_0) -\bar{\theta}(t; \theta_0)\| + \|\theta(t; \theta_0) - \widehat{\theta}_{-j}\|\\
    \leq & C\left[\log\frac{1}{\|\theta_0 -  \widehat{\theta}_{-j}\|}\right]^2 \cdot\|\theta_0 -  \widehat{\theta}_{-j}\|^2,
\end{align}
here we take $t = O\left(\log\frac{1}{\|\theta_0 -  \widehat{\theta}_{-j}\|}\right)$.

By the result from the first step that $\|\widehat{\theta}_{\rmj}^{\DP} -\widehat{\theta}_{\rmj}\| \leq \sqrt{s(\cA)}\leq  \sqrt{s(\cA)} + s(\cA)$ and $\|\widehat{\theta}_{n}^{\DP} -\widehat{\theta}_{\rmj}\| \leq  \sqrt{s(\cA)} + s(\cA)$ w.h.p, as well as the fact that 
 $\widetilde{\theta}_{n,\rmj} = g_{\rmj}(\widehat{\theta}_{n}^{\DP})$ and 
$\widetilde{\theta}_{\rmj,\rmj} = g_{\rmj}(\widehat{\theta}_{\rmj}^{\DP})$, we have
\begin{equation}
    \|\widetilde{\theta}_{n,\rmj} - \widetilde{\theta}_{\rmj,\rmj}\| \leq C\log^2\left(\frac{1}{s(\cA)+ \sqrt{s(\cA)}}\right)\left(s(\cA) + \sqrt{s(\cA)}\right)^2.
\end{equation}

\textbf{Step 3. Complete the statement}

With probability at least $1- 2e^{-\epsilon/\sqrt{s(\cA)}}$, we have
\begin{align*}
    &|f(x;~ \widetilde{\theta}_{n,\rmj}) -
    f(x;~ \widetilde{\theta}_{\rmj,\rmj})| 
    \leq \Lip(f)  \|\widetilde{\theta}_{n,\rmj} - \widetilde{\theta}_{\rmj,\rmj}\|\\
    \leq & C\Lip(f)\log^2\left(\frac{1}{s(\cA)+ \sqrt{s(\cA)}}\right)\left(s(\cA) + \sqrt{s(\cA)}\right)^2.
\end{align*}

Hence, suppose $s(\cA) = O(1/n)$, we have with probability at least $1- e^{-c \epsilon \sqrt{n}}$ for some $c>0$, 
\[
    |f(x;~ \widetilde{\theta}_{n,\rmj}) -
    f(x;~ \widetilde{\theta}_{\rmj,\rmj})| = O\left(\frac{\log^2(n)}{n} \right).
    \]

\end{proof}

\subsection{Local Strong Convexity of Neural Network}\label{sec:local_convexity}

Researchers have found that for the parameter recovery setting of neural networks, recovery guarantees can be provided if the neural network is initialized around the ground truth, due to the local strong convexity in the neighborhood of the ground truth parameters, see \cite{zhong2017recovery}. Also, \cite{li2017convergence} shows that in general, the convergence of NN parameters can be split into two phases. Once the parameters fall into a $\delta$-one point strong convex regime near the ground truth, the convergence changes to phase II, leading to fast parameter convergence to the ground truth.

\begin{lemma}\label{lemma:local_convexity}
(Positive Definiteness of Hessian near the ground truth in one-hidden-layer neural network \citep{zhong2017recovery})
If $(x_i, y_i), i\in [n]$ are sampled i.i.d.\ from the distribution: $x\sim \cN(0, I_d),~ y = \sum_{i=1}^k v_i^* \phi({w_i^*}^\top x)$ with $\phi$ being the activation functions as ReLU, sigmoid or tanh, and $k$ being the number of nodes in the hidden layer,  we can show that  as long as $\| W - W^* \| \leq \text{poly}(1/k, 1/\gamma, \frac{\nu_{\min}}{\nu_{\max}}) \| W^*\|$,  if $n> d \cdot \text{poly}(\log d, t, k, \nu, \gamma, \frac{\sigma_1^{2p}}{\rho})$, where $\nu_{\min} = \min_i |\nu^*|; \nu_{\max} = \max_i \nu^*$ we have with probability at least $1-d^{\Omega(t)}$,
    \begin{equation}
        \Omega(\nu_{\min} \rho(\sigma_k)/ (\kappa^2\gamma)) I \preceq \nabla^2 f(\bX, W) \preceq O(k\nu_{\max}^2 \sigma_1^{2p})I,
    \end{equation}
where $\sigma_i$ is the $i$-th singular value of $W^*$, $\kappa=\sigma_1/\sigma_k$, $\gamma= \prod_{i=1}^k/\sigma_k^k$, and $\rho(\sigma_k)>0$ (defined in \cite{zhong2017recovery}) is related to the activation function.
\end{lemma}

\section{Appendix. Additional Experiment Results}

In this section, we include experiment results in different settings. 

\subsection{Effect of stopping time when training the neural networks}

To look into the effects of stopping time of the full-data model, we increase the maximum number of epochs when training  neural network models in all three mentioned methods (including all leave-one-out models in the jackknife+ method). 

We generate the data in the same way as described in 
 \cref{sec:simulation}
 with  dimension  $p=100$, \ie
$X_i \sim \cN(0, 5\cdot\mathbb{I}_{p}),~ i\in [N]$, where $N=5000$, and 
\begin{equation}
    Y_i = \sqrt{\ReLU(X_i^\top \beta)} + \epsilon_i, \text{ where } \epsilon_i \sim \cN(0, 0.5), ~
    \beta \sim \text{Beta(1.0, 2.5)}.
\end{equation}
We construct the prediction intervals using jackknife+, lazy finetune and DP-Lazy PI using the training and calibration with $n=100$ and $\alpha=0.1$. When training the neural network with $(64, 64)$ hidden layers, we increase the maximum number of epochs from $10$ (as in  \cref{sec:simulation}) to $20$ to avoid the early stopping of the NN training, while we keep the other parameters unchanged: $\lambda = 10$ and batch size $= 10$.

\begin{figure}[ht]
    \includegraphics[width=\textwidth]{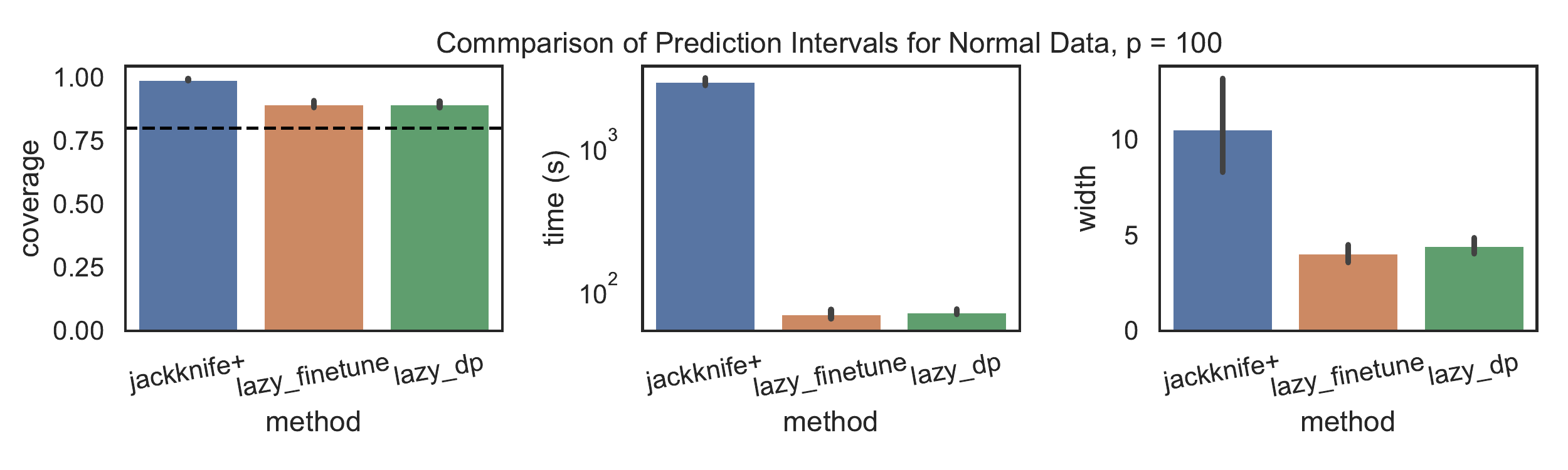}
    \caption{The average coverage, average computing time and width of the prediction intervals with $\alpha=0.1$ in 15 trials, with the error bars to show the $\pm$ standard errors on the simulation data with the dimension $p=100$. The dotted line in the coverage panel indicates the target coverage level $80\%$. When training the (64, 64) hidden layer neural network, we take $\lambda=10$, batch size = $10$ and \textbf{max \# epochs = $20$}. The vertical axis in second panel is presented in log scale. }\label{fig:larger_epoch}
\end{figure}

 The results in \cref{fig:larger_epoch} demonstrate that lazy methods reliably maintain a coverage above $80\%$, exhibit significantly faster computation times compared to jackknife+, and yield more accurate prediction intervals. Moreover, an examination of Figure 2(b) reveals that increasing the maximum number of epochs negatively impacts jackknife+'s performance while enhancing the accuracy of lazy methods. Since lazy methods require training the full model only once, increasing the stopping time has minimal impact on computation time. In contrast, Jackknife+ experiences a dramatic increase in computing time due to the multiplication of training time by the training sample size. This disparity makes the computation more prohibitive for jackknife+.
 In terms of the interval width, the results depicted in the right panel \cref{fig:larger_epoch} highlight the superior interval width accuracy achieved by lazy methods.

 This result also reconciles the performance of simulated data and real data experiments that for high-dimension data, jackknife+ tends to give a much wider and more unstable prediction interval than DP-Lazy PI.

 We also investigate the effect of increasing the maximum number of epochs on real data sets, specifically the BlogFeedback data set and Medical Expenditure Panel Survey 2016 data set. The results demonstrate a similar phenomenon as depicted in Figure \ref{fig: real data epoch20}, highlighting extended training epochs on the performance of the models make the jackknife+ prediction intervals worse(wider).
 
 \begin{figure}[ht]
    \centering
    \begin{tabular}{{@{}c@{}}}
        \includegraphics[width=0.95\textwidth]{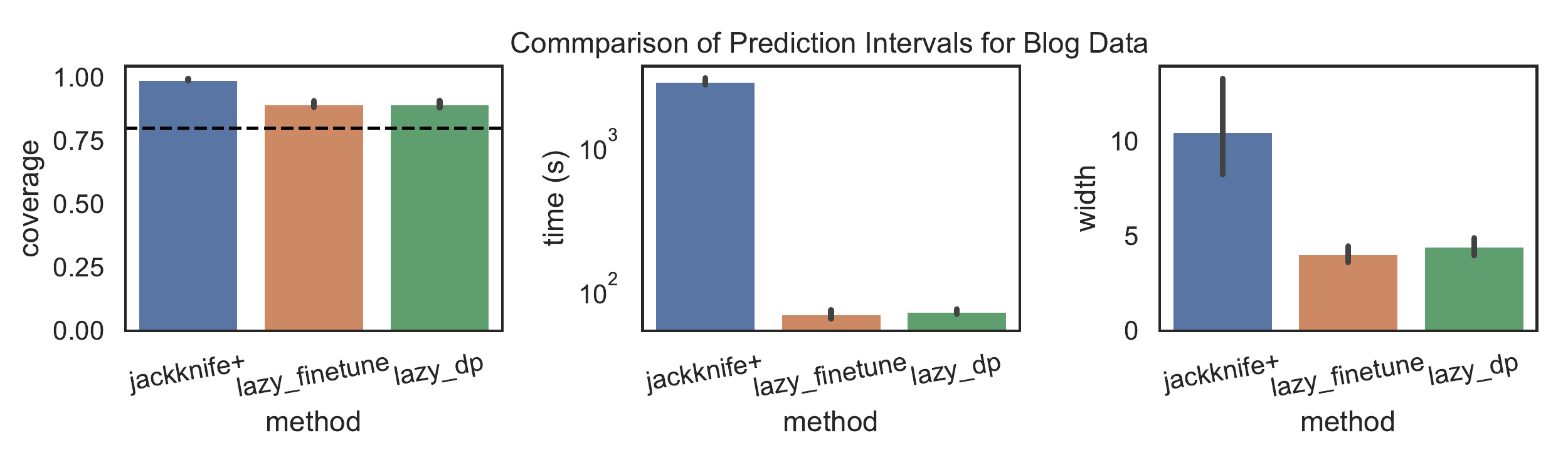}\\
        \small{(a) BlogFeedback data set, max \#epochs=20}
    \end{tabular}
    \begin{tabular}{{@{}c@{}}}
    \includegraphics[width = 0.95\textwidth]{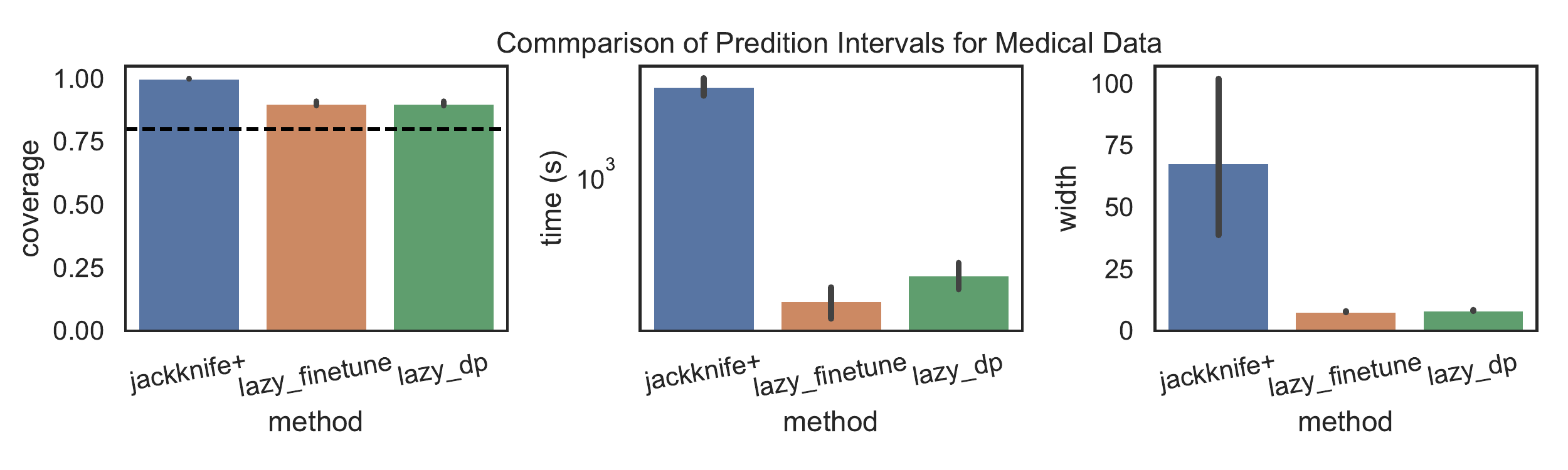}\\
    \small{(b) Medical Expenditure Panel Survey 2016 data set, max \#epochs=20}
    \end{tabular}
    \caption{The coverage, average computing time and width of the prediction intervals on real data sets with $\alpha = 0.1$. The dotted line in the left panels in (a) and (b) is the target $80\%$ coverage. Other parameters remain the same as Figure 3 except the max epoch changes from 10 to 20.
    }\label{fig: real data epoch20}
\end{figure}

\subsection{Effect of batch size when training the neural networks}

All the experiments in the main paper have the batch size $n_{\text{batch}} = 10$ given the training size $n=100$ when training the neural networks. In this section, we give the real data experiment results when the batch size is set to be $n_{\text{batch}}=1$ and compare the results from different methods.

As the batch size set to be 1 when training the neural network, \cref{fig: real data batch1} shows that jackknife+ gets worse (wider) prediction intervals compared to Figure 3, due to the more severe overfitting problem.

\begin{figure}[ht]
    \centering
    \begin{tabular}{{@{}c@{}}}
        \includegraphics[width=0.95\textwidth]{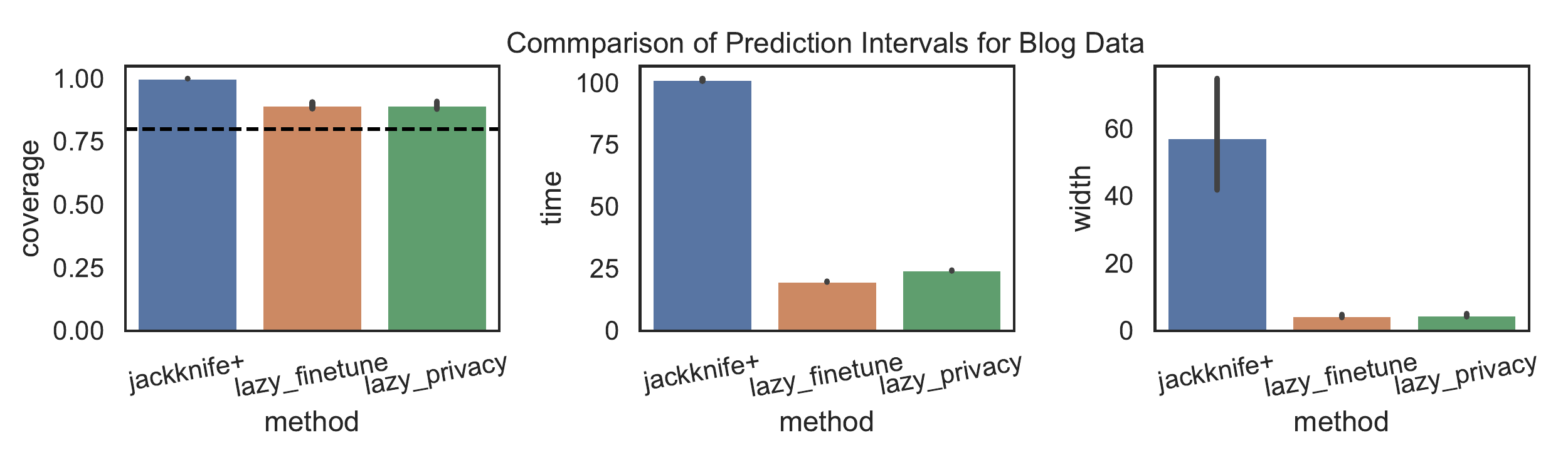}\\
        \small{(a) BlogFeedback data set, batch size=1}
    \end{tabular}
    \begin{tabular}{{@{}c@{}}}
    \includegraphics[width = 0.95\textwidth]{figure/med_e10_b10.pdf}\\
    \small{(b) Medical Expenditure Panel Survey 2016 data set, batch size=1}
    \end{tabular}
    \caption{The coverage, average computing time and width of the prediction intervals on real data sets with $\alpha = 0.1$. The dotted line in the left panels in (a) and (b) is the target $80\%$ coverage. Other parameters remain the same as Figure 3 except the batch size changes to 1 and max number of epochs changes to 2.
    }\label{fig: real data batch1}
\end{figure}

\end{document}